\documentclass[letterpaper]{article} % DO NOT CHANGE THIS
\usepackage{aaai2027}
\nocopyright
\usepackage[hyphens]{url}  % DO NOT CHANGE THIS
\usepackage{graphicx} % DO NOT CHANGE THIS
\usepackage{natbib}  % DO NOT CHANGE THIS AND DO NOT ADD ANY OPTIONS TO IT
\usepackage{caption} % DO NOT CHANGE THIS AND DO NOT ADD ANY OPTIONS TO IT
\usepackage{algorithm}
\usepackage{algorithmic}
\usepackage{newfloat}
\usepackage{listings}
\DeclareCaptionStyle{ruled}{labelfont=normalfont,labelsep=colon,strut=off} % DO NOT CHANGE THIS
\floatstyle{ruled}
\newfloat{listing}{tb}{lst}{}
\floatname{listing}{Listing}
\usepackage{booktabs}
\usepackage{amsmath}
\usepackage{amssymb}
\usepackage{amsfonts}
\usepackage{amsthm}
\theoremstyle{plain}
\newtheorem{proposition}{Proposition}
\theoremstyle{definition}
\newtheorem{assumption}{Assumption}

\newcommand{\topk}{top-$k$}
\newcommand{\R}{\mathbb{R}}

\newcommand{\sig}{\sigma}
\newcommand{\Var}{\mathrm{Var}}

\title{When Can You Debias an LLM Judge?\\ Identifiability Limits, a Test, and Designs for Top-$k$ Ranking}
\author{
    Jian Xu\textsuperscript{\rm 1,\rm 2},
    Delu Zeng\textsuperscript{\rm 3},
    John Paisley\textsuperscript{\rm 4},
    Qibin Zhao\textsuperscript{\rm 2}
}
\affiliations{
    \textsuperscript{\rm 1}RIKEN iTHEMS\quad
    \textsuperscript{\rm 2}RIKEN AIP\quad
    \textsuperscript{\rm 3}South China University of Technology\quad
    \textsuperscript{\rm 4}Columbia University\\
    jian.xu@riken.jp
}

\begin{document}

\maketitle

\begin{abstract}
Large language models (LLMs) are increasingly used as cheap, scalable judges that compare
candidate outputs pairwise. Because such judges prefer verbose or well-formatted answers, the
natural fix is to add bias covariates to a Bradley--Terry model and estimate the bias away. We
show this cannot work as advertised: the quality/bias split is \emph{not identified} by pairwise
comparisons, and the failure is exact---across $48$ real judge-pools the profile likelihood over
the coefficient is flat to $\mathbf{0.0000}$ \textbf{nats}, and scaling the comparisons
$26\times$ buys none. A ``debiased'' score is selected by the prior, not recovered from data.
Our contribution is accordingly not a better estimator but a characterization of \emph{when
prior-based correction is justified}, plus designs that supply the missing information when it
is not. The assumption
the prior encodes---quality is a priori uncorrelated with the covariate---pays only while
$\mathrm{corr}(\theta,x)$ stays below a crossing point (configuration-dependent, $0.22$--$0.60$),
which is what makes the same model help on LLMBar and hurt on SummEval and Nectar. We give two
escapes: a \textbf{trusted-anchor gate} that decides per (judge, covariate, task) (no false enables in $6{,}000$
decisions at $K\ge10$ anchors, a rate our sample bounds at $\le6\%$), and a \textbf{paired
rendering design}---the only design-based intervention \emph{we study} that restores likelihood identification, turning
$0.0000$ nats into $172.7$ and making a prior-free estimate exist. A
secondary \topk-aware acquisition heuristic beats budget-agnostic and D-optimal allocation
under both a frozen and a stochastic judge, but its edge over Thompson/LUCB is specific to the
frozen protocol. Across fifteen real LLM judges bias is heterogeneous and capability-dependent:
correction improves \topk{} recall by $0.20$--$0.32$ on five biased-but-competent cheap judges and is
a no-op on frontier ones (Spearman $\rho{=}{-}0.84$ between competence and gain over the $14$
competent judges, $p{<}10^{-3}$), concentrating the benefit where at-scale evaluation happens.
\end{abstract}

% \begin{links}
%     \link{Code}{https://anonymous.4open.science/r/bayes-judge}
% \end{links}

\section{Introduction}
Pairwise comparison by an LLM ``judge'' has become a default tool for evaluating and
selecting natural-language outputs: ranking chatbot responses, scoring model variants,
filtering generated candidates, and triaging documents. The appeal is cost---an LLM
judge is orders of magnitude cheaper than a human annotator---but the cost comes with
two well-documented liabilities. LLM judges are \emph{noisy}, returning inconsistent
verdicts across paraphrases and runs, and they are \emph{systematically biased},
preferring longer, more elaborate, or better-formatted answers and favoring whichever
candidate appears in a particular position \citep{zheng2023judging,wang2023fair,
zeng2024llmbar,dubois2024length}. The consequence is subtle but damaging: if we simply
count wins, the resulting ranking reflects \emph{presentation}---verbosity and
formatting---as much as it reflects quality.

In most applications we do not need the full ranking; we need the \topk{} items: the
best few responses to keep, the strongest models to promote, the papers worth a human
read. This reframing has two implications that prior work on LLM judging has not
jointly addressed. First, recovering the correct \topk{} under a biased judge is a
\emph{statistical estimation} problem---and, as we show, one the comparison data alone
cannot solve: quality and the judge's presentation preference are perfectly confounded, so
any separation rests on an assumption we must state and test. Second, comparisons cost
money, so \emph{which} pairs we query matters: effort spent distinguishing items far
from the \topk{} boundary is wasted.

We address both (Fig.~\ref{fig:overview}). We model judging as a Bayesian Bradley--Terry
process over latent item quality, augmented with \emph{bias covariates}---observable item
features such as elaboration, together with a position term---whose coefficients are fit
jointly with quality. This is the model a practitioner would naturally reach for, and our
first finding is about its limits rather than its performance: the shrinkage prior on the
covariate coefficient does not merely regularize, it \emph{selects} the quality/bias split,
because the comparisons themselves carry exactly zero information about that split. Reporting
its output as ``debiased quality'' is thus a category error---the number is a restatement of
the prior's assumption, and it is right precisely when that assumption is. What can be done
honestly is to make the assumption explicit, map the regime where it pays, assess its operational validity per (judge, covariate, task)
with a few trusted labels, and---where the evaluator controls how items are rendered---design
the comparison set so the split becomes estimable at all. That is the arc of this paper. On
top of the resulting posterior we additionally define a \topk-aware acquisition function that
selects the next comparison to most reduce uncertainty about \topk{} \emph{membership}. We
treat this as a secondary, applied contribution rather than a headline: it is a heuristic, its
advantage over budget-agnostic allocation is solid but its advantage over standard
information-theoretic and bandit rules is not, and we report the regimes where it vanishes as
carefully as the ones where it holds.

Our position is therefore not that we estimate and remove judge bias, but that
\emph{we characterize when prior-based correction is justified, and supply identifiable or
supervised designs when it is not}. Concretely:
\begin{itemize}
\item An \textbf{exact non-identifiability result}: the item-level covariate coefficient is not
identified by comparisons (only apparent quality $\theta{+}cx$ and the position term are), and
this is not a small-sample artifact---on all $48$ real judge-pools the profile likelihood over
$c$ is flat to $0.0000$ nats, unchanged as comparisons scale $26\times$. ``Debiasing'' is thus
an assumption the prior encodes, not a quantity the data reveal.
\item A \textbf{characterization of when the assumption pays}: the benefit falls monotonically
with $\mathrm{corr}(\theta,x)$ and reverses at a configuration-dependent crossing
($0.22$--$0.60$), predicting the split between benchmarks where correction helps and hurts.
\item \textbf{Two designs that escape the assumption}: a trusted-anchor \textbf{gate} that
tests it per (judge, covariate, task) with a few labels, and a \textbf{paired rendering design} that restores
identifiability by construction---the only design-based intervention we study that restores
likelihood identification.
\item A secondary, applied \textbf{\topk-aware acquisition heuristic}: it beats budget-agnostic
and D-optimal allocation, though its edge over Thompson/LUCB is specific to a frozen judge.
\item \textbf{Evidence on fifteen real LLM judges}: verbosity bias is prevalent but
\emph{capability-dependent}---strong on cheap judges, absent on frontier ones---so it must be
tested per judge, not assumed.
\end{itemize}

\begin{figure*}[t]
\centering
\includegraphics[width=0.99\textwidth]{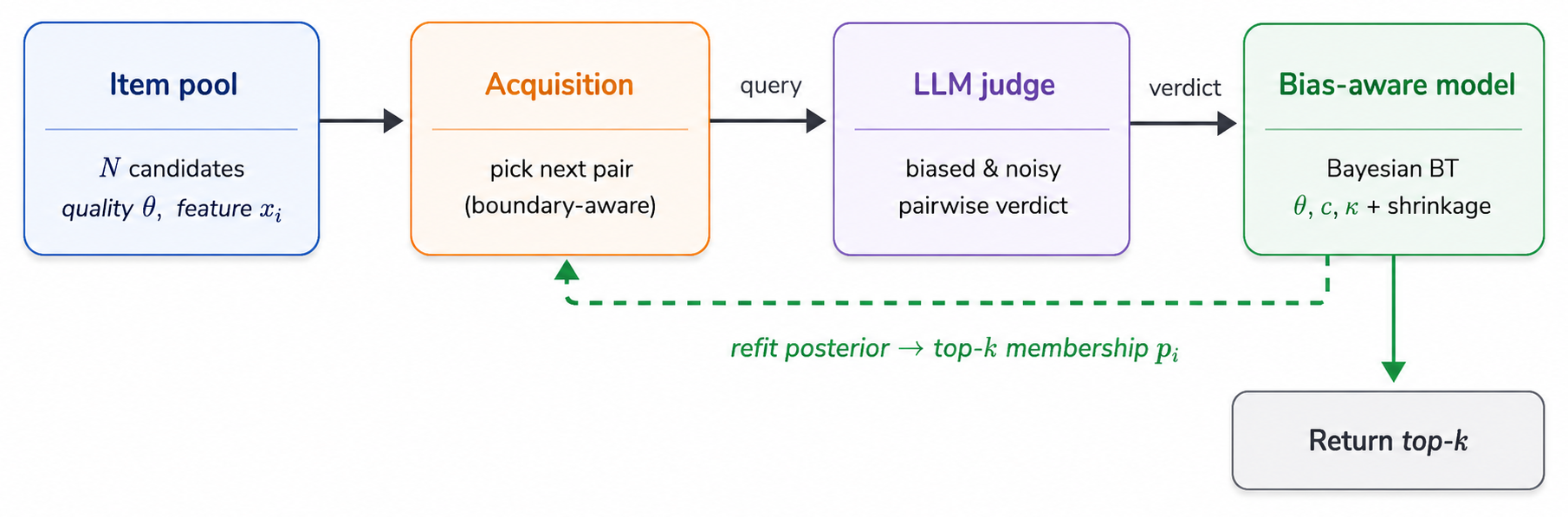}
\caption{Overview. A budget-limited active loop: the \topk-aware rule picks the next pair for
a biased, noisy LLM judge; the verdict updates a bias-aware Bayesian Bradley--Terry posterior
over quality $\theta$ and bias $(c,\kappa)$; the refreshed \topk{}-membership probabilities
$p_i$ drive the next pick. The $\theta$/$c$ split is an \emph{assumption} the prior selects
(Prop.~\ref{prop:ident}).}
\label{fig:overview}
\end{figure*}

\section{Related Work}
\paragraph{LLM-as-a-judge and its biases.}
LLM judges are now standard for open-ended evaluation
\citep{zheng2023judging,li2024llmjudgesurvey}, with well-documented position/order
\citep{wang2023fair}, verbosity/length \citep{dubois2024length,saito2023verbosity} and
self-preference biases; LLMBar \citep{zeng2024llmbar} shows judges are fooled by superficially
appealing but worse answers. Mitigations are largely heuristic (order-swapping, calibration
prompts, majority vote). Placing bias inside a generative model and estimating it is the
natural next step and is the model we study---but our contribution is to show it cannot identify
the item-level coefficient, so the ``correction'' is the prior's assumption, not a measured
property of the judge (the position term excepted, identified by within-pair order variation).

\paragraph{Bayesian methods around LLMs.}
A growing line treats LLM components probabilistically
\citep{tonolini2024bayespe,ross2025textual,huang2026bayesagent}. Closest is BAL-PM
\citep{melo2024balpm}, Bayesian active learning that selects informative preference pairs for
reward-model training; our target is \topk{} \emph{identification} under a \emph{biased} judge,
with the bias modeled explicitly, rather than reward-model data efficiency.

\paragraph{Identification in measurement models.}
The impossibility behind Prop.~\ref{prop:ident} is not new, and we say so plainly: granting
each item a free effect makes an item-level covariate collinear with the item block, a
close relative of the fixed-effects/incidental-parameter problem
\citep{neyman1948consistent,lancaster2000incidental}, and the standard remedy---conditioning on
or designing for within-unit variation \citep{chamberlain1980analysis,andersen1970asymptotic}---is
exactly what our paired rendering design does. Related ideas recur in psychometrics
\citep{linacre1989many,holland2012differential,meredith1993measurement}. We claim no theorem
that would surprise a psychometrician; our contribution is that this known impossibility
silently governs LLM-judge debiasing, which is performed and reported as if the coefficient
were estimated. The supplementary material expands the correspondence.

\paragraph{Ranking from comparisons and active selection.}
Recovering quality from pairwise outcomes is the Bradley--Terry/Thurstone problem
\citep{bradley1952rank}. Active and budgeted variants---dueling bandits, active ranking,
and top-$k$ identification---choose comparisons adaptively
\citep{szorenyi2015online,heckel2019active}. Our acquisition rule is a \topk-aware,
bias-aware instance of Bayesian experimental design \citep{houlsby2011bald}: it uses the
posterior over \emph{membership} in the \topk{} set, and it is computed under a judge
model that contains nuisance bias parameters.

\section{Problem Setup and Model}
\paragraph{Setup.}
We have $N$ items with unknown scalar quality $\theta \in \R^N$. A query presents an
ordered pair $(a,b)$ to an LLM judge, which returns $y=1$ if it prefers the
\emph{first-shown} item $a$ and $y=0$ otherwise. Each item $i$ carries observable
covariates $x_i \in \R^d$ that may bias the judge. Throughout the controlled benchmark
$x_i$ is a \emph{single binary elaboration indicator}, $x_i\in\{0,1\}$, recording whether item
$i$ was rendered tersely or with elaboration clauses; we use the binary indicator rather than
raw length because our benchmark manipulates elaboration directly and holds the statement
count fixed, so the indicator is the manipulated variable and length is its consequence. (The
external-data study of \S\ref{sec:exp} instead uses standardized word count, which is the only
covariate available there; we say so explicitly at each such use.) Given a budget of $B$ comparisons we want to return the set of
the $k$ highest-quality items, $S_k(\theta)=\arg\!\topk_i \theta_i$.

\paragraph{Bias-aware Bradley--Terry model.}
We model the judge's verdict as
\begin{equation}
\Pr(y=1 \mid a,b) = \sig\!\Big(\theta_a - \theta_b + \textstyle\sum_{m=1}^{d} c_m (x_{a,m}-x_{b,m}) + \kappa\Big),
\label{eq:model}
\end{equation}
where $\sig$ is the logistic function, $c \in \R^d$ are \emph{bias coefficients} on the
item covariates (e.g.\ a positive $c$ for length encodes verbosity bias), and $\kappa$
is a \emph{position} term capturing a constant preference for the first-shown answer
(its sign gives the direction of position bias). Setting $c=0,\kappa=0$ recovers the
ordinary Bradley--Terry model used by naive win-counting; we call that the \emph{naive}
model and Eq.~\eqref{eq:model} the \emph{bias-aware} model.

\paragraph{Priors and adaptive shrinkage.}
We place $\theta \sim \mathcal{N}(0,\lambda^{-1}I)$ and a shrinkage prior on the bias
coefficients, $c_m,\kappa \sim \mathcal{N}(0,\tau^2)$ with a finite, moderate $\tau$ (our $\lambda_b{=}\tau^{-2}{=}0.1$ is $\tau\!\approx\!3.2$, not a tight prior). It is tempting to
describe this as letting the data decide which biases are active, and that description is
wrong: by Prop.~\ref{prop:ident} the comparisons determine only the apparent quality
$\phi=\theta+cx$ and the position term $\kappa$, never the item-level split between
$\theta$ and $c$. What the data can decide is $\kappa$---a genuine per-comparison effect,
identified because presentation order varies \emph{within} an item pair---and the overall
fit. What the prior decides is the decomposition of $\phi$, via
Eq.~\eqref{eq:cstar}. A smaller $\tau$ does not detect that a judge is unbiased; it
\emph{assumes} more strongly that the covariate-aligned part of apparent quality is real
quality, and the estimate follows.

Two consequences we state plainly rather than paper over. First, when a judge is truly
unbiased the shrinkage prior does keep the model close to ordinary Bradley--Terry, but ``close''
is not ``free'': on our unbiased synthetic judge the bias-aware model reaches $0.79$ \topk{}
recall against naive's $0.86$---a $7$-point drop from estimating a coefficient that should be
zero, which is a real cost, not a rounding error. Second, a hierarchical or empirical-Bayes
$\tau$ does \emph{not} rescue identifiability: tuning $\tau$ on the marginal likelihood
selects adaptively among prior-induced decompositions, and every one of them is a
different assumption about $\theta \perp x$, not a measurement of the judge. Nothing in the
hyperprior can recover a direction the likelihood is flat along. A horseshoe prior is a
drop-in alternative with the same caveat. We use a fixed moderate $\tau$ in the main results
and report a sensitivity sweep in the appendix.

\section{Analysis}
We give three results that explain the empirical findings: naive aggregation is
\emph{inconsistent} for \topk{} recovery under a biased judge; the quality/bias split in the
bias-aware model is \emph{non-identifiable}, so the prior rather than the data selects it, and
a paired rendering design restores identifiability; and the acquisition
score~(S2) is a boundary-targeted surrogate rather than an exact
information rule. The first and third are unsurprising instantiations of omitted-variable bias
and a delta-method-style heuristic. The second is the paper's centre of gravity: it is what
makes the natural fix inadmissible as stated, and everything downstream---the gate, the paired
design, the $\mathrm{corr}(\theta,x)$ characterization---exists to answer it. Proof sketches
are given here; details are in the appendix.

\paragraph{Naive aggregation recovers the wrong \topk{}.}
We state the cleanest result in the linear (Thurstone) comparison model, where the judge
prefers $a$ iff $\theta_a-\theta_b+c^\star(x_a-x_b)+\varepsilon>0$ with symmetric noise
$\varepsilon$ and randomized order; the logistic case behaves analogously.

\begin{assumption}\label{as:indep}
The bias covariate is non-degenerate and uncorrelated with quality across items,
$\tfrac1N\sum_i x_i\theta_i \to 0$, and $c^\star\neq 0$.
\end{assumption}

\begin{proposition}[Confounded estimand]\label{prop:naive}
Let $\tilde\theta$ be the population minimizer of the naive objective
(Eq.~\eqref{eq:obj} with $c=\kappa=0$). Then $\tilde\theta_i=\theta_i+c^\star x_i$ up to
an additive constant. Consequently, for any items $i,j$ with $\theta_i>\theta_j$ but
$\theta_i-\theta_j< c^\star(x_j-x_i)$, the naive ranking places $j$ above $i$; whenever
such a pair straddles the \topk{} boundary, $S_k(\tilde\theta)\neq S_k(\theta)$, and no
amount of additional comparisons removes the error.
\end{proposition}
\noindent\emph{Sketch.} Omitting $x$ from a correctly specified linear comparison model
is an omitted-variable problem: because each item appears with its own indicator, the
naive fit assigns to item $i$ the quantity that best explains its win rate, namely
$\theta_i+c^\star x_i$. Under Assumption~\ref{as:indep} the $c^\star x_i$ term does not
average out at the item level, so it shifts the per-item estimand and reorders items
whose quality gap is smaller than their verbosity gap. The shift is a property of the
estimand, not of finite samples, giving the budget-independent plateau in
Table~\ref{tab:main}. \hfill$\square$

\paragraph{The quality/bias split is \emph{not} identified by comparisons alone.}
\begin{proposition}[Non-identifiability of the item-level bias coefficient]\label{prop:ident}
With a free quality parameter $\theta_i$ per item and a fixed item-level covariate $x_i$,
the likelihood of Eq.~\eqref{eq:model} is invariant under
\begin{equation}
\theta_i' = \theta_i + \delta x_i,\qquad c' = c-\delta,\qquad \forall\delta\in\R .
\label{eq:invariance}
\end{equation}
Hence no set of comparisons---within-$x$, cross-$x$, in any number---identifies $c$ or
$\theta$ separately; only the \emph{apparent quality} $\phi_i=\theta_i+c\,x_i$ (up to an
additive constant) and the position term $\kappa$ are identified.
\end{proposition}
\noindent\emph{Proof.} Substituting~\eqref{eq:invariance} into the logit gives
$(\theta_a{+}\delta x_a)-(\theta_b{+}\delta x_b)+(c{-}\delta)(x_a{-}x_b)+\kappa
=\theta_a-\theta_b+c(x_a-x_b)+\kappa$, so every comparison probability is unchanged.
Equivalently, the design column for $c$, whose entries are $x_a-x_b$, equals
$\sum_i x_i\,(\text{item-indicator column }i)$ and therefore lies in the span of the
$\theta$ columns. Numerically, the design matrix for our $N{=}30$ pools has rank $30$ rather
than $32$: two null directions, the usual Bradley--Terry constant and
exactly~\eqref{eq:invariance}. $\kappa$ \emph{is} identified, because order randomization
makes its column independent of the item block. \hfill$\square$

\paragraph{What the prior does---and the assumption it encodes.}
The Gaussian priors make the MAP unique, but they \emph{choose} a point on the
unidentified ray~\eqref{eq:invariance} rather than letting the data locate it. The
likelihood is flat along two directions---\eqref{eq:invariance} and the Bradley--Terry
constant $\theta\mapsto\theta+\alpha\mathbf{1}$---so a candidate reconstruction is
$\theta=\phi-c\,x+\alpha\mathbf{1}$, and the prior selects \emph{both} free parameters. Profiling
$\alpha$ out of $\tfrac{\lambda}{2}\|\phi-c\,x+\alpha\mathbf{1}\|^{2}+\tfrac{1}{2\tau^{2}}c^{2}$
gives $\alpha^\star=c\bar{x}-\bar{\phi}$, hence $\theta=(\phi-\bar\phi\mathbf{1})-c\,(x-\bar
x\mathbf{1})$ and
\begin{equation}
\hat c_{\mathrm{MAP}}=\frac{\lambda\,\langle \phi-\bar\phi\mathbf{1},\; x-\bar x\mathbf{1}\rangle}
             {\lambda\|x-\bar x\mathbf{1}\|^2+\tau^{-2}}
      \;=\;\frac{\lambda N\,\widehat{\mathrm{Cov}}(\phi,x)}{\lambda N\,\widehat{\mathrm{Var}}(x)+\tau^{-2}},
\label{eq:cstar}
\end{equation}
where the centering is not cosmetic: $x$ is a binary elaboration indicator with mean
$\approx\!\tfrac12$, and using the uncentered inner products misstates $\hat c_{\mathrm{MAP}}$ by up to a
factor of five in our pools. The right reading is the second form: \textbf{the prior attributes
to bias precisely the part of apparent quality that \emph{co-varies} with the covariate},
shrunk by $\tau$. Our estimator is therefore not ``separating quality from
presentation using the data''; it is applying the assumption that \textbf{true quality is a
priori uncorrelated with the presentation covariate}, and reading off the corresponding
decomposition. This assumption is not directly observable without trusted labels; our gate instead tests its
operational consequence---whether the induced correction improves agreement on a small anchor
set. When it
holds the estimate is right, and when it fails the correction removes real signal. The
empirical sweep in Fig.~\ref{fig:corr} confirms this: the benefit of the correction decreases
monotonically with $\mathrm{corr}(\theta,x)$ and changes sign at a point that, in our
simulation setting, sits near $0.45$---which is why the method helps on LLMBar's adversarial
pairs and hurts on SummEval and Nectar. Eq.~\eqref{eq:cstar} also explains \emph{why} that
quantity is the governing one: $\hat c_{\mathrm{MAP}}$ tracks $\widehat{\mathrm{Cov}}(\phi,x)$, so the
correction subtracts real quality exactly to the extent that quality and covariate co-vary.
Restoring genuine identifiability requires extra structure---e.g.\ rendering the same
content at two verbosity levels so a shared $\theta_i$ appears with two different $x$
values, which breaks~\eqref{eq:invariance}; we verify this design below.

\section{Experiments}
\label{sec:exp}
\paragraph{Controlled benchmark.}
To measure \topk{} recovery we need known ground truth, so we construct items whose true
quality is controlled. Each item is an answer to a fixed factual question, assembled from
a bank of \emph{true} and \emph{false} statements; with a fixed statement count per item,
quality $\theta_i$ is the number of correct minus incorrect statements. Independently of
quality we assign a \emph{verbosity} level---terse statements vs.\ the same statements
with an added authoritative-sounding (content-free) elaboration clause---so that length
is decoupled from quality by construction ($\mathrm{corr}(\theta,\text{verbosity})\approx
0$). We use $N{=}30$, $k{=}5$, and each pool is built to contain \emph{exactly} $k{=}5$
fully-correct items ($\theta{=}6$, all statements true), separated by a one-point gap from the
rest ($\theta_{(5)}{=}6 > \theta_{(6)}$ in every pool), so the true \topk{} set is unique and
free of boundary ties. \topk{} recall is the fraction of the $5$ returned items that lie in this
set; equivalently our tie-robust estimator counts item $i$ when $\hat\theta_i$ ranks it in the
top $5$ and its \emph{true} $\theta_i\ge\theta_{(5)}$, which here coincides with membership in
the fully-correct set.

\paragraph{Judges.}
We collect the full matrix of ordered pairwise verdicts from \textbf{fifteen} real LLM
judges spanning open and proprietary families and a wide capability range: open-weight
\textbf{Llama-3.1-8B}, \textbf{Qwen-2.5-Coder-7B}, \textbf{Phi-4} and \textbf{DeepSeek-Coder-6.7B}
(served locally); and API judges \textbf{GPT-4o-mini/5.1},
\textbf{Gemini-2.5-flash-lite/pro}, \textbf{DeepSeek-chat/V4-flash/V4-pro}, and
\textbf{Claude Haiku-4.5/Sonnet-4.6/Opus-4.8}. Collecting the full oracle once lets us
evaluate every acquisition strategy offline against the same judge, with seeds controlling
the strategy's randomization. Both presentation orders are collected for every unordered pair
($N(N{-}1)$ ordered queries); acquisition methods then select among the resulting ordered
queries, so the rule---not a coin flip---decides which item is shown first.

\paragraph{Do real judges exhibit the modeled biases?}
Before the ranking task we probe each bias in isolation with equal-quality pairs. The
verdict is unambiguous: GPT-4o-mini prefers the elaborated answer in $100\%$ of
length-contrasted pairs of equal content, and a markdown-formatted answer in $100\%$ of
format-contrasted pairs, while showing a strong second-position preference. Plain
length \emph{padding} with obvious filler, by contrast, is penalized---verbosity bias is
about apparent substance, not raw length, which is why our covariate is an elaboration
indicator.

\begin{table*}[t]
\centering
\small
\begin{tabular}{lcccccc}
\toprule
Judge & $n$ & $\hat c$ (verbosity) & $\hat\kappa$ (position) & naive recall & bias-aware recall & $\Delta$ \\
\midrule
\multicolumn{7}{l}{\emph{Biased but competent judges --- bias-aware modeling recovers the \topk{}}}\\
Llama-3.1-8B           & 10 & $+1.78$ & $+0.35$ & $0.50$ & $\mathbf{0.82}$ & $+0.32$ \\
Qwen-2.5-Coder-7B      & 10 & $+1.45$ & $+1.08$ & $0.64$ & $\mathbf{0.86}$ & $+0.22$ \\
Gemini-2.5-flash-lite  & 3 & $+1.32$ & $-1.30$ & $0.60$ & $\mathbf{0.87}$ & $+0.27$ \\
DeepSeek-chat          & 5 & $+1.00$ & $-2.40$ & $0.80$ & $\mathbf{1.00}$ & $+0.20$ \\
GPT-4o-mini            & 5 & $+0.98$ & $-0.84$ & $0.56$ & $\mathbf{0.76}$ & $+0.20$ \\
\midrule
\multicolumn{7}{l}{\emph{Frontier / near-neutral judges --- naive already accurate; correction has small effects (mostly negligible, occasionally a small downside)}}\\
DeepSeek-V4-flash      & 3 & $+0.63$ & $-0.08$ & $1.00$ & $1.00$ & $\;\;\,0.00$ \\
Claude-Sonnet-4.6      & 3 & $+0.46$ & $+1.42$ & $1.00$ & $1.00$ & $\;\;\,0.00$ \\
DeepSeek-V4-pro        & 2 & $+0.44$ & $+0.01$ & $1.00$ & $1.00$ & $\;\;\,0.00$ \\
GPT-4.1                & 3 & $+0.38$ & $+1.18$ & $1.00$ & $1.00$ & $\;\;\,0.00$ \\
Phi-4                  & 10 & $+0.34$ & $+0.32$ & $0.86$ & $\mathbf{0.92}$ & $+0.06$ \\
Gemini-2.5-pro         & 1 & $+0.22$ & $-0.31$ & $1.00$ & $1.00$ & $\;\;\,0.00$ \\
GPT-5.1                & 3 & $+0.14$ & $+0.95$ & $1.00$ & $1.00$ & $\;\;\,0.00$ \\
Claude-Opus-4.8        & 1 & $+0.05$ & $+1.36$ & $1.00$ & $1.00$ & $\;\;\,0.00$ \\
Claude-Haiku-4.5       & 3 & $-0.50$ & $+0.05$ & $0.93$ & $0.87$ & $-0.07$ \\
\midrule
\multicolumn{7}{l}{\emph{Judge too weak to be useful --- little quality signal to recover}}\\
DeepSeek-Coder-6.7B    & 5 & $-0.98$ & $+1.74$ & $0.12$ & $0.08$ & $-0.04$ \\
\bottomrule
\end{tabular}
\caption{Fifteen LLM judges, full-data fits with the binary elaboration covariate
($\lambda_b{=}0.1$), sorted by the length-bias coefficient $\hat c$; $\Delta$ is the bias-aware
model's \topk{} recall gain over naive. \emph{Reading $\hat c$:} since the $\theta$/$c$ split is
prior-selected (Prop.~\ref{prop:ident}), $\hat c$ is not a likelihood-identified constant but
the covariate-aligned component of apparent quality (Eq.~\eqref{eq:cstar}), comparable across
rows only because all share the same pools, prior and covariate.}
\label{tab:main}
\end{table*}

\paragraph{Q1: bias modeling recovers the correct \topk{} (correctness).}
Table~\ref{tab:main} gives the central result. The biased judges
(Llama, Qwen, Gemini, DeepSeek, GPT-4o-mini) have positive verbosity coefficients and
naive estimates that correlate with verbosity ($+0.22$ to $+0.49$): naive aggregation is
partly ranking by presentation (Fig.~\ref{fig:boundary}, appendix). There the naive model's \topk{} recall is stuck
and---critically---does \emph{not improve with more comparisons} (it converges to the wrong
set), while the bias-aware model, which attributes the covariate-aligned component of apparent
quality to bias (under the prior of Eq.~\eqref{eq:cstar}), lifts recall by $+0.20$ to $+0.32$, reaching $0.76$--$1.0$, a gap that persists at
full data. The gain tracks the bias magnitude: it shrinks to $+0.06$ for the nearly-neutral
Phi-4, and correction slightly \emph{hurts} the mildly brevity-preferring Claude-Haiku
($0.93\!\to\!0.87$, $\hat c{=}{-}0.50$), a small but honest reminder that modeling a covariate
is not free when the judge barely has that bias. This is the intended pattern: the benefit
holds because the covariate is spurious in these pools; when it is not, correction actively
hurts
(Fig.~\ref{fig:corr}), which is what the gate is for.

\paragraph{Significance and baselines.}
Table~\ref{tab:sig} (appendix) puts the correctness result on a statistical footing, with the
same care about the unit of analysis as the acquisition study. The independent unit is the
\textbf{pool} (one item draw), and---importantly---the \emph{same} pools are reused across
judges (seed $s$ gives identical items), so pool-level differences are not independent across
judges. We therefore test each judge separately at the pool level with an exact paired
permutation test, Holm-corrected across the five judges. The gain is significant on the two
judges with $10$ pools (Llama $+.32$, $p{=}.010$; Qwen $+.22$, $p{=}.016$) and \emph{underpowered
elsewhere}: at $n{=}3$--$5$ pools an exact test cannot reach $p{<}.05$ at all, so the
$+.20$--$.27$ gains on Gemini-flash, GPT-4o-mini and DeepSeek-chat are point estimates we do not
call significant---even where a bootstrap CI excludes zero, which at $n{=}5$ is
anti-conservative. For a single pooled statement we take the \emph{judge} as the cluster
($5$ judge-mean gains), giving $+.24$ with $t$-test $p{=}.0005$; this is far more conservative
than pooling all pool-differences as if independent (which the shared-pool structure forbids)
and is the number we stand behind. On baselines: \emph{order-swap majority vote}, the standard
position-debias, does not help ($.61$, vs.\ naive $.61$) because it leaves verbosity
uncorrected, while \emph{post-hoc length residualization} matches the bias-aware model at full
data ($.85$ vs.\ $.85$). We therefore claim no full-data accuracy advantage over a well-chosen
residualization heuristic; the joint model's value is that it (i)~handles verbosity \emph{and}
position in one estimator, (ii)~supplies the calibrated posterior the active rule needs, and
(iii)~admits a tunable shrinkage that reduces, though does not eliminate, the downside on
unknown judges---none of which a post-hoc
residual provides.

\paragraph{Q2: \topk-aware acquisition (secondary).}
Given a usable posterior, our \topk-aware rule targets the membership boundary rather than the
full ranking. Tested at the pool level (seeds within a pool share its frozen verdict matrix and
are not independent), it beats every baseline---Thompson and LUCB included---on the two judges
with enough pools to support a test (Holm $p\le.023$), and leads on the rest where $3$--$5$
pools cannot reach significance. Its edge over Thompson/LUCB is specific to the frozen protocol
(under a stochastic judge it still beats round-robin and D-optimality but no longer separates
from the bandits). As a heuristic with narrow scope we treat it as secondary and defer the
budget curves, baselines, stochastic-judge robustness, score ablation and $2{\times}2$ synergy
to the supplementary material.

\paragraph{External validity on LLMBar.}
Our controlled benchmark isolates spurious verbosity by construction; LLMBar
\citep{zeng2024llmbar} tests the same phenomenon in the wild, with human-annotated pairs
whose \emph{adversarial} subsets are built so the worse answer is superficially
attractive. We judge all $419$ pairs with GPT-4o-mini in both orders and ask whether
modeling the presentation features recovers the gold preference. Fig.~\ref{tab:llmbar}
shows the judge is genuinely length-biased on the adversarial subsets---it prefers the
longer answer far more often than gold does (e.g.\ Neighbor: prefers-longer $0.40$ vs.\
gold-longer $0.13$)---whereas on the Natural subset length is legitimate ($0.55$ vs.\
$0.56$). Correcting for length and format (a cross-validated logistic model on the judge
score plus length/format differences, the pairwise analogue of our covariate model)
improves agreement with gold on exactly the biased subsets---Neighbor $0.58\!\to\!0.86$,
Manual $0.63\!\to\!0.80$---while leaving the unbiased Natural subset unchanged
($0.94\!\to\!0.93$). This is direct evidence that modeling the spurious feature recovers
ground truth when the feature is in fact spurious. We are explicit about what this does
\emph{not} show: the correction here is fitted using \emph{supervision} from gold labels, so it
validates the mechanism---the judge's length preference is spurious on these pairs and
conditioning on length recovers the human verdict---rather than the unsupervised top-$k$
pipeline. Given Prop.~\ref{prop:ident}, that separation is the honest reading: with gold
labels the split is identified by supervision; without them it is chosen by the prior.

\begin{figure}[t]
\centering
\includegraphics[width=0.86\columnwidth]{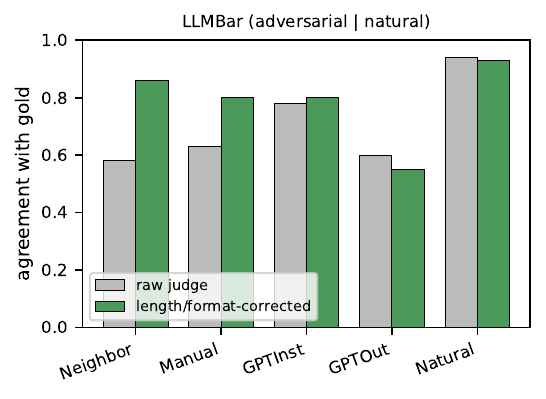}
\caption{External validity on LLMBar (judge: GPT-4o-mini, both orders). On the adversarial
subsets---where the judge prefers the longer answer far more than gold does (e.g.\ Neighbor
$0.40$ vs.\ $0.13$)---a length/format correction recovers gold agreement (Neighbor
$0.58\!\to\!0.86$); the unbiased Natural subset is unchanged. This is the positive
counterpart to the SummEval limitation, where length legitimately signals quality.}
\label{tab:llmbar}
\end{figure}

\paragraph{Restoring identifiability: a paired rendering design.}
Prop.~\ref{prop:ident} says the flat direction exists because each item carries both its own
$\theta_i$ and a single fixed $x_i$. Rendering the \emph{same} content at two verbosity levels
removes it: a shared $\theta_i$ then appears at two covariate values, so a within-content
comparison has the $\theta$ terms cancel and its logit reduces to
$c\,(x_{\text{verbose}}-x_{\text{terse}})+\kappa$---a direct read on $c$. We built such pools
($15$ base contents $\times$ terse/verbose $=30$ items) and collected the full verdict matrix
from \emph{four} open judges $\times$ $6$ pools each ($24$ pools, $20{,}880$ verdicts), which
lets us compare the two \emph{analyses of identical data}: an unpaired model that gives every
rendering its own $\theta$, versus the paired model that ties the two renderings of a base to
one $\theta$.

The structural fix is unambiguous and reproduces on every pool. The unpaired design carries
\textbf{two} null directions and the paired design exactly \textbf{one}---the unavoidable
Bradley--Terry constant---in $24/24$ pools across all four judges. The decisive measurement is
the \emph{profile} likelihood $\max_{\theta,\kappa}\ell(\theta,c,\kappa)$ at fixed $c$, which
is the right test because non-identifiability means $\theta$ can absorb any shift in $c$:
under the unpaired design, moving $c$ by $0.5$ and re-optimizing $\theta$ costs
$\mathbf{0.000}$ \textbf{nats} on all $24$ pools---the data are not merely uninformative about
$c$, they are \emph{exactly} silent---whereas under the paired design the profile is curved
and a prior-free maximum-likelihood estimate of $c$ exists.
(An earlier version of this test that shifted $c$ \emph{without} re-optimizing $\theta$
reports a large likelihood drop even in the unpaired design; that number measures nothing, as
the compensating $\theta$ shift is exactly what the invariance provides.)

\paragraph{No amount of data substitutes for the design.}
A natural hope is that non-identifiability is really a small-sample problem that more
comparisons would cure. It is not, and the paired design lets us show this sharply. We swept
the pool size over $8, 15, 25, 40$ base contents ($4$ judges $\times$ $3$ seeds each; the
largest pools are $80$ items and $6{,}320$ ordered comparisons) and measured the profile
likelihood range over $c\in[\hat c{-}1,\hat c{+}1]$---the total evidence the data carry about
the split. In the unpaired design it is $\mathbf{0.0000}$ \textbf{nats at every size}, in
$48/48$ pools: multiplying the comparisons $26\times$ buys \emph{exactly zero} information
about $c$ (Table~\ref{tab:paired}, top). In the paired design the same sweep gives
$10.8 \to 30.8 \to 75.3 \to 172.7$ nats. Identifiability is a property of \emph{what you
compare}, not of \emph{how much} you compare, and this is the practical reason the paired
design matters: it is the only one of our interventions that restores likelihood
identification rather than adding an assumption.

\paragraph{What kind of prior dependence remains.}
The paired estimate is still not prior-\emph{free} in practice, and it is worth being precise
about which of two very different dependencies survives. Sensitivity of $\hat c$ to its own
prior $\tau$ falls $3.2\times$ ($0.402\!\to\!0.124$, $p{<}10^{-12}$), but sensitivity to the
\emph{quality} prior $\lambda$ falls only $1.5\times$ even at $40$ bases. That residue,
however, is ordinary ridge attenuation rather than the pathology. The decisive check is that
in the paired design a prior-free maximum-likelihood estimate $\hat c_{\text{MLE}}$
\emph{exists}, and a weak-prior MAP recovers it to within $0.092$ on average
($\lambda{=}0.2$, $12$ pools at $40$ bases); increasing $\lambda$ then pulls the estimate
monotonically toward zero ($0.30$ at $\lambda{=}1$, $0.66$ at $\lambda{=}5$)---the textbook
signature of shrinkage bias, which shrinks with $\lambda$ and vanishes with $n$. The
per-judge pattern confirms the mechanism: DeepSeek-Coder, whose verdicts are nearly a
function of elaboration alone and whose profile is therefore sharply curved, barely moves at
all ($-1.234\!\to\!-1.211$ as $\lambda$ goes from $0.2$ to $5$), whereas Llama, for whom $c$
is a small correction to a mostly quality-driven verdict, moves most. In the unpaired design
this comparison cannot even be posed: the profile is exactly flat, no $\hat c_{\text{MLE}}$
exists, and $100\%$ of the estimate is the prior's choice. So the paired design does not
merely reduce prior dependence---it converts a dependence no amount of care can remove into
one a practitioner controls by not over-shrinking $\theta$.

\begin{table}[t]
\centering\small
\setlength{\tabcolsep}{4pt}
\begin{tabular}{lcccc}
\toprule
\multicolumn{5}{l}{\textbf{(a) Evidence the data carry about $c$} (profile range, nats)}\\
base contents & $8$ & $15$ & $25$ & $40$ \\
\midrule
unpaired design & $\mathbf{.0000}$ & $\mathbf{.0000}$ & $\mathbf{.0000}$ & $\mathbf{.0000}$ \\
paired design & $10.8$ & $30.8$ & $75.3$ & $\mathbf{172.7}$ \\
\midrule
\multicolumn{5}{l}{\textbf{(b) Is the residue the pathology, or just shrinkage?}}\\
\multicolumn{5}{l}{$|\hat c_{\text{MAP}}(\lambda) - \hat c_{\text{MLE}}|$, paired, $40$ bases}\\
$\lambda$ & \multicolumn{1}{c}{$0.2$} & \multicolumn{1}{c}{$1.0$} & \multicolumn{1}{c}{$5.0$} & \\
\midrule
deviation from prior-free MLE & $\mathbf{.09}$ & $.30$ & $.66$ & \\
\multicolumn{5}{l}{\emph{unpaired}: $\hat c_{\text{MLE}}$ does not exist---profile exactly flat} \\
\bottomrule
\end{tabular}
\caption{The paired design adds information; more data does not. \textbf{(a)}~Profile-likelihood
range over $c$ (evidence the comparisons carry about the split); \textbf{(b)}~$\hat c$'s
deviation from its prior-free MLE across prior strengths $\lambda$. Discussed in the text.}
\label{tab:paired}
\end{table}

\begin{figure}[t]
\centering
\includegraphics[width=0.86\columnwidth]{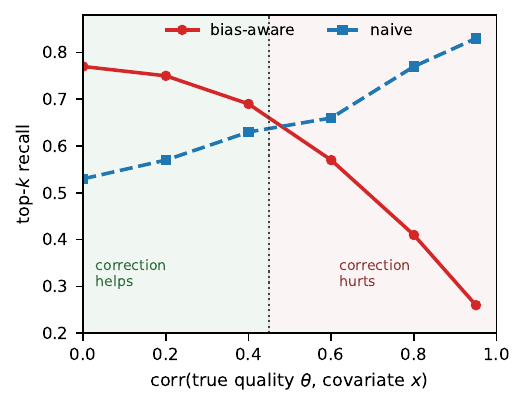}
\caption{The prior assumption governs the outcome. Sweeping the true $\mathrm{corr}(\theta,x)$
in a synthetic judge ($c^\star{=}1$, $30$ runs/point), the correction's benefit falls
monotonically and changes sign near $0.45$---helping when the covariate is spurious and hurting
when it tracks quality.}
\label{fig:corr}
\end{figure}

\paragraph{Gating: when to enable correction.}
Because a covariate can be spurious (LLMBar) or legitimate (SummEval), the practitioner needs
an operational rule for whether to turn correction on. We propose a \emph{trusted-anchor
gate}: obtain a small set of ground-truth pairwise labels (from a strong judge or a few human
comparisons) and enable the covariate correction only if the bias-aware model agrees with
these anchors at least as well as naive. Evaluating a safety claim demands care about what a
run of zeros can support, so we depart from a single anchor draw on one judge: we resample the
anchor set $100$ times per pool across $5$ judges and $60$ pools ($6{,}000$ gate decisions per
anchor budget), and report Clopper--Pearson intervals with the \emph{pool} as the independent
unit (Table~\ref{tab:gate}).

The gate is strongly safety-biased but not infallible. False enables are rare and confined to
the smallest budget---$2$ of $6{,}000$ decisions at $K{=}5$ (touching $2$ of $60$ pools), none
at $K\ge10$---so our earlier claim that $5$ anchors suffice for safety does not survive the
stronger test. Two honest qualifications follow. First, zero \emph{observed} failures is not a
zero rate: $0/60$ pools bounds the pool-level false-enable rate at $\le6.0\%$
(Clopper--Pearson), not at $0$. What holds is that across five judges the gate never enabled on
a legitimate covariate at any $K\ge10$, the behavior that avoids the SummEval/Nectar harm.
Second, capturing the \emph{gain} costs labels: the spurious-pool enable-rate climbs from
$22\%$ ($K{=}5$) to $59\%$ ($K{=}40$), and $40$ anchors---$9\%$ of an $N{=}30$ pool's pairs,
$\approx$\$$0.14$ at strong-judge prices---is about an order of magnitude more than the cheap
judge's own $\approx$\$$0.014$ evaluation (Fig.~\ref{fig:cost}, appendix) for a single pool, but amortizes
across reused pools. The
reading is: the gate buys real protection, its residual risk is bounded at $\le6\%$ rather than
shown to be zero, and recovering most of the gain costs about an order of magnitude more than the
cheap-judge evaluation for a single pool, but amortizes across reused pools.

\begin{table}[t]
\centering\small
\setlength{\tabcolsep}{3pt}
\begin{tabular}{lccccc}
\toprule
anchors $K$ & $5$ & $10$ & $20$ & $40$ & $60$ \\
\midrule
enable-rate, spurious pools & $22\%$ & $34\%$ & $49\%$ & $\mathbf{59\%}$ & $51\%$ \\
enable-rate, legitimate pools & $0\%$ & $0\%$ & $0\%$ & $0\%$ & $0\%$ \\
\midrule
false enables / $6{,}000$ decisions & $2$ & $\mathbf{0}$ & $\mathbf{0}$ & $\mathbf{0}$ & $\mathbf{0}$ \\
pools with $\ge1$ false enable & $2/60$ & $0/60$ & $0/60$ & $0/60$ & $0/60$ \\
\;\;CP $95\%$ upper bound & $11.5\%$ & $6.0\%$ & $6.0\%$ & $6.0\%$ & $6.0\%$ \\
\midrule
gated recall, spurious & $.67$ & $.70$ & $.73$ & $\mathbf{.76}$ & $.74$ \\
\;\;(oracle-on / best possible) & \multicolumn{5}{c}{$.85$ / $.85$} \\
gated recall, legitimate & $.61$ & $.61$ & $.61$ & $.61$ & $.61$ \\
\bottomrule
\end{tabular}
\caption{Trusted-anchor gate: anchor sets resampled $100\times$ per pool across $5$ judges and
$60$ pools ($6{,}000$ decisions per column). Enable-rate on spurious/legitimate pools,
false-enable count with Clopper--Pearson bound, and gated recall vs.\ anchor budget $K$.}
\label{tab:gate}
\end{table}

\paragraph{Q3: where does the method help?}
The bias-aware gain $\Delta$ is strongly negatively correlated with a judge's raw competence
(naive recall) across the $14$ judges with quality to recover: Spearman $\rho{=}{-}0.84$
($p{<}10^{-3}$; partly a ceiling artifact, since a judge at recall $1.0$ can only have
$\Delta{=}0$). We exclude the one \emph{sub-competent} judge (Coder-6.7B, naive recall $0.12$)
from this correlation: it has almost no quality signal, so neither model recovers a \topk{}
and its $\Delta{=}{-}0.04$ is a small artifact rather than a failure of the trend. The method
thus helps most on judges competent enough to rank but corrupted by bias, and is a no-op on
those that already rank perfectly. The
practically relevant regime is clear: frontier judges \emph{in our set} rank well but cost
$10$--$100\times$ more per comparison, while the cheap judges at-scale evaluation actually uses
are the ones our model helps (Fig.~\ref{fig:capability}, appendix).

\section{Discussion and Conclusion}
Bias across our fifteen judges is \emph{heterogeneous}---common but neither universal nor
capability-independent, with position bias flipping direction---and because each $\hat c$ is
prior-selected, Table~\ref{tab:main} supports a \emph{relative} reading across judges, not an
absolute coefficient. The larger point is that ``debiasing'' an LLM judge from comparisons
alone is not estimation but an assumption the prior selects; the useful questions are when it
holds---which we characterize via $\mathrm{corr}(\theta,x)$---and what to do when it does not:
gate on trusted anchors, or use a paired design that makes the split estimable. A per-judge-bias
\emph{multi-fidelity} escalation, spending a strong judge only at the \topk{} boundary, is a
natural next step; naively mixing two judges in a single-bias model backfires, since their
biases differ.

\bibliography{aaai2027}

\newpage
\appendix
\section*{Technical Appendix}

\subsection*{Limitations}
\paragraph{Limitations.}
Four caveats bound our claims. (i)~\emph{The covariate must capture a spurious preference.}
Our controlled benchmark isolates content-free verbosity; on data where length legitimately
signals quality this assumption breaks. On SummEval---real machine summaries with human
quality scores---raw summary length correlates with the human score ($\hat c{=}{+}0.79$),
and debiasing on it \emph{hurts} \topk{} recovery (naive $0.52$ vs.\ bias-aware $0.43$ over
$20$ articles). This is the mirror image of the LLMBar result (Fig.~2), where
length is spurious and modeling it helps: the method is warranted precisely when the
covariate captures a presentation preference that is plausibly \emph{not} a quality signal.
A second real multi-candidate top-$k$ task, Nectar ($7$ answers per prompt with GPT-4
rankings), behaves like SummEval: length legitimately tracks the ground-truth ranking (the
GPT-4 ranker itself rewards detail; $\hat c{=}{+}0.53$), so blind debiasing again slightly
hurts (top-$2$ recall $0.58\!\to\!0.52$). The pattern is consistent---the method helps only
when the ground truth does \emph{not} share the judge's presentation preference (LLMBar's
human gold), and hurts when it does (SummEval, Nectar). The trusted-anchor gate above makes
this decision operational, avoiding the SummEval/Nectar-type harm while keeping the
LLMBar-type gain.
(ii)~\emph{The acquisition advantage is protocol-dependent.} Our main protocol caches one
verdict per ordered pair, which isolates systematic bias but is not how an API judge
behaves ($43\%$ of pairs are genuinely uncertain under the judge's own logits; at
temperature $1$, GPT-4o-mini flips $8\%$ of verdicts). Under a stochastic oracle with
re-queries ($10$ pools), our rule still beats round-robin and D-optimality but no longer
separates from Thompson sampling or LUCB at any budget, so the part of its advantage that is
specific to strong bandit baselines belongs to the repeat-free protocol rather than to the
rule. The bias-correction results, which use the full matrix, are unaffected by this caveat.
(iii)~\emph{Small samples for frontier judges.} Gemini-2.5-pro and Opus-4.8 ($n{=}1$ each), and V4-pro
($n{=}2$) have few oracles; their perfect-recall entries mean ``no bias detected in a small
sample,'' not a strong guarantee. (iv)~Statements about frontier judges being unbiased hold
only for the specific models and prompt we tested.

\subsection*{Deferred method details and experiments}
\begin{figure}[t]
\centering
\includegraphics[width=0.86\columnwidth]{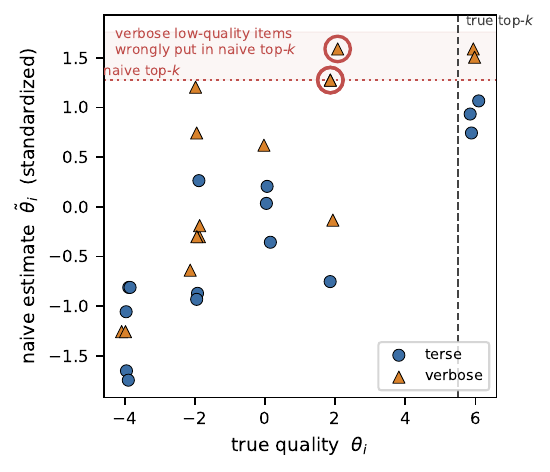}
\caption{Why naive aggregation fails (Llama judge, full data). Each item is placed at its
true quality $\theta_i$ (x) and its naive quality estimate $\tilde\theta_i$ (y). Verbosity
inflates the naive estimate, so verbose low-quality items (circled) rise above the naive
\topk{} cut-off (dotted) while a terse truly-top item falls below it: the naive ranking
selects on presentation, realizing Prop.~\ref{prop:naive}. The bias-aware model removes the
vertical verbosity shift and recovers the true \topk{} (right of the dashed line).}
\label{fig:boundary}
\end{figure}
\begin{table}[t]
\centering\small
\setlength{\tabcolsep}{3pt}
\begin{tabular}{lccccc}
\toprule
Biased judge & $n$ & naive & bias-aw. & $\Delta$ & boot.\ 95\% CI \\
\midrule
Llama-3.1-8B  & 10 & $.50$ & $.82$ & $+.32^{*}$ & $[+.24,+.40]$ \\
Qwen-2.5-Coder-7B & 10 & $.64$ & $.86$ & $+.22^{*}$ & $[+.16,+.28]$ \\
Gemini-flash  & 3  & $.60$ & $.87$ & $+.27^{\dagger}$ & $[+.20,+.40]$ \\
GPT-4o-mini   & 5  & $.56$ & $.76$ & $+.20^{\dagger}$ & $[+.08,+.32]$ \\
DeepSeek-chat & 5  & $.80$ & $1.00$ & $+.20^{\dagger}$ & $[+.04,+.40]$ \\
\midrule
\multicolumn{6}{l}{\emph{Debiasing baselines (mean over the biased judges):}}\\
order-swap+maj. & \multicolumn{5}{c}{$.61$ \; (vs.\ naive $.61$)}\\
length-resid. & \multicolumn{5}{c}{$.85$ \; ($\approx$ bias-aware $.85$)}\\
\bottomrule
\end{tabular}
\caption{Correctness with statistics, \textbf{pool} as the unit of analysis. $\Delta$ is the
per-pool bias-aware$-$naive gain (mean; bootstrap $95\%$ CI over pools). $p$-values are exact
paired permutation tests, Holm-corrected across the five judges: $^{*}p{<}.05$ (Llama
$.010$, Qwen $.016$). $^{\dagger}$: \textbf{underpowered}---with $n$ pools an exact test cannot
return $p{<}2/2^{n}$, i.e.\ $.063$ at $n{=}5$ and $.25$ at $n{=}3$, so significance is
unobtainable, not absent (point estimates favor bias-aware on all five). DeepSeek-chat is the
case to read carefully: its bootstrap CI excludes $0$ yet its exact $p$ is $.25$---the
percentile bootstrap is anti-conservative at $n{=}5$ while the exact test cannot reject, so we
treat it as not established. \emph{Pooled} across judges is delicate because all judges reuse
the \emph{same} item pools (seed $s\!\to\!$ identical items), so pool differences are not
independent across judges; taking the \textbf{judge} as the cluster ($5$ judge-mean $\Delta$s)
gives mean $+.24$, $t$-test $p{=}.0005$---far more conservative than a naive
oracle-level pooling would suggest. Order-swap majority vote (the standard position-debias)
leaves verbosity uncorrected; length residualization matches bias-aware at full data---its
advantage is the joint estimate, calibrated uncertainty for acquisition, and adaptive
shrinkage.}
\label{tab:sig}
\end{table}
\begin{figure}[t]
\centering
\includegraphics[width=0.99\columnwidth]{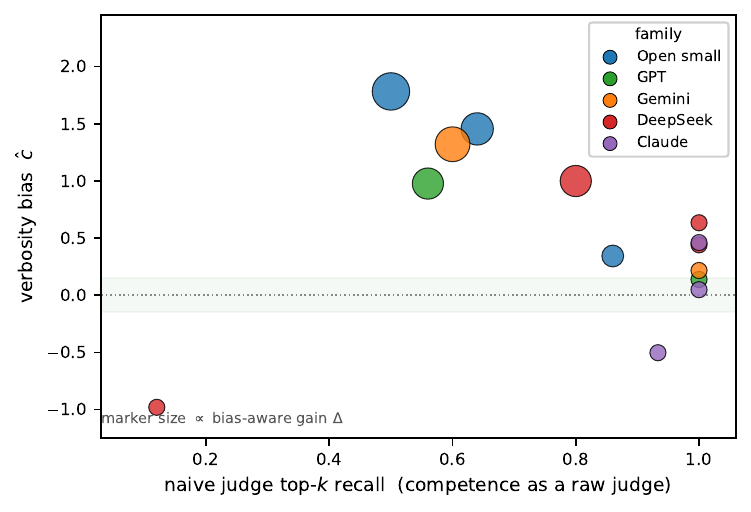}
\caption{Where the method helps (our tested judges). Each point is a judge, colored by
model family, at its naive \topk{} recall (x, its raw competence) and verbosity coefficient
$\hat c$ (y); marker size grows with the bias-aware gain $\Delta$. The gain concentrates on
competent-but-biased judges (top-left, large markers; Spearman $\rho{=}{-}0.84$ over the
$14$ competent judges, between naive recall and $\Delta$); the frontier judges we tested already rank perfectly with $\hat
c\!\approx\!0$, and the weakest judge (bottom-left) has too little quality signal to help.
Within-family trends are visible (e.g.\ Gemini flash-lite $\to$ 2.5-pro). As in
Table~\ref{tab:main}, $\hat c$ is a prior-selected decomposition, not a
likelihood-identified judge constant (Prop.~\ref{prop:ident}); the $y$-axis is comparable
across judges only because every point uses the same pools, prior and covariate, and should
not be read as an absolute bias measurement.}
\label{fig:capability}
\end{figure}

\begin{figure}[t]
\centering
\includegraphics[width=0.92\columnwidth]{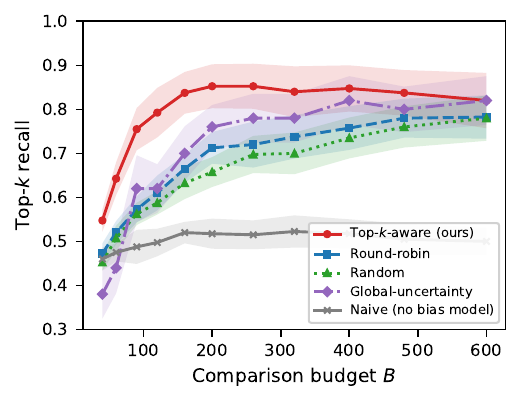}
\caption{\topk{} recall vs.\ comparison budget (Llama-3.1-8B judge). Curves are means over
$10$ pools $\times\,8$ acquisition seeds. \textbf{Error bars are $\pm$s.e.\ across the $10$
pools}, computed after averaging seeds within each pool---seeds share a pool's frozen oracle,
so a seed-level s.e.\ would be pseudo-replicated and roughly $\sqrt{8}$ too narrow. Curve
separation here is therefore descriptive; the pool-level tests are in
Table~\ref{tab:acqbase}. The \topk-aware rule reaches high recall with the fewest comparisons;
global-uncertainty is worst at small budgets; the naive model (no bias term) plateaus at a
wrong \topk{} regardless of budget.}
\label{fig:budget}
\end{figure}
\begin{figure}[t]
\centering
\includegraphics[width=0.86\columnwidth]{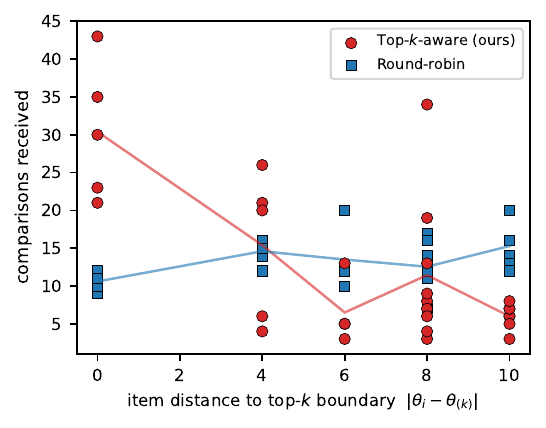}
\caption{How the budget is spent (Llama judge, $B{=}200$). \topk-aware acquisition
concentrates comparisons on items near the \topk{} boundary ($\sim$$3\times$ more than
round-robin at distance $0$) and starves items far from it, whereas round-robin spreads
effort uniformly. This is the mechanism behind the budget savings in
Fig.~\ref{fig:budget}.}
\label{fig:acq}
\end{figure}
\begin{figure}[t]
\centering
\includegraphics[width=0.82\columnwidth]{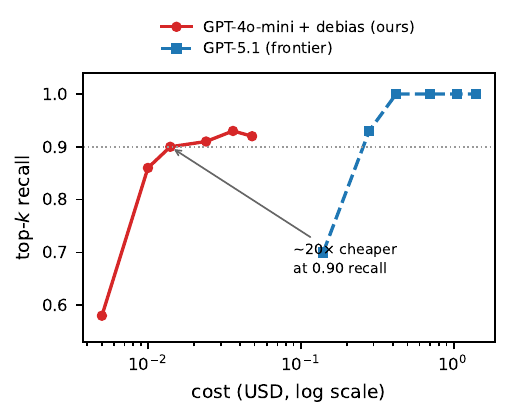}
\caption{Cost--quality frontier \emph{for this workload and this pricing snapshot}. A cheap
biased judge (GPT-4o-mini) with our correction reaches $0.89$ recall (at its best budget $B\!\approx\!200$) for
$\approx$$20\times$ less than a frontier judge (GPT-5.1). The ratio is an instance, not a
constant: it is computed from our prompts (a comparison averages $\approx$$293$ input tokens, range
$165$--$427$, estimated at $1.3$ tokens/word, with \texttt{max\_tokens}$=16$ on output and no
reasoning tokens) at July 2026 list prices, and it moves with
prompt length, output budget, vendor pricing, and whether the frontier judge spends hidden
reasoning tokens---Gemini-2.5-pro, which spends $\approx$$500$, is far costlier per verdict
than its list price suggests. Longer answers or a reasoning-heavy cheap judge would compress
the gap; the qualitative point (frontier judges buy the last few points of recall at a large
multiple) is what we rely on.}
\label{fig:cost}
\end{figure}

\paragraph{What the covariates capture.}
The vector $x_i$ collects the presentation attributes a judge might reward independently
of quality. Our controlled experiments use this single binary elaboration indicator, but the same slot
admits response length in tokens, a formatting indicator (markdown/structure), or a
self-preference feature (whether a candidate comes from the judge's own model family);
each adds one coefficient $c_m$ and one column to the design. Because Eq.~\eqref{eq:model}
depends only on the \emph{difference} $x_a-x_b$, the model is a logistic regression of the
verdict on paired quality-indicator and covariate-difference features---which is what
makes estimation convex and gives each $c_m$ a clean reading as the number of preference
logits that one unit of the feature buys. The position term $\kappa$ is a special
covariate that is a property of the \emph{query} (which slot an item occupies) rather than
of the item; collecting both presentation orders makes it separable from quality
(see~\emph{Inference}).

\paragraph{Inference.}
Write $\beta=(\theta,c,\kappa)$ and let $z_{ab}=(e_a-e_b,\,x_a-x_b,\,1)$ be the design
vector of a comparison, so Eq.~\eqref{eq:model} is $\Pr(y=1)=\sig(\beta^\top z_{ab})$.
Given revealed comparisons $\mathcal{D}$ we minimize the penalized negative
log-posterior
\begin{equation}
\begin{aligned}
\mathcal{L}(\beta)= &-\!\!\sum_{(a,b,y)\in\mathcal{D}}\!\!\big[y\log\sig(\beta^\top z_{ab})+(1{-}y)\log\sig(-\beta^\top z_{ab})\big]\\
&+\tfrac{\lambda}{2}\|\theta\|^2+\tfrac{1}{2\tau^2}(\|c\|^2+\kappa^2),
\end{aligned}
\label{eq:obj}
\end{equation}
which is convex in $\beta$; we solve it with L-BFGS (warm-started across acquisition
steps). The Laplace posterior is $\beta\approx\mathcal{N}(\hat\beta,H^{-1})$ with
$H=\nabla^2\mathcal{L}(\hat\beta)=\sum_{\mathcal{D}}\pi_{ab}(1{-}\pi_{ab})\,z_{ab}z_{ab}^\top+\Lambda$,
where $\pi_{ab}=\sig(\hat\beta^\top z_{ab})$ and $\Lambda$ is the prior precision; the
quality block of $H^{-1}$ gives $\Sigma$. This is an $(N{+}d{+}1)$-dimensional convex
problem and yields both the point ranking and the uncertainty the acquisition rule
needs. The Monte-Carlo \topk{}-membership probabilities $p_i$ are read off $S$ samples
from $\mathcal{N}(\hat\theta,\Sigma)$. A refit therefore costs, consistent with the L-BFGS solver above,
$O\!\big(T\,|\mathcal{D}|\,(N{+}d)\big)$ for the $T$ optimization steps, plus a \emph{single}
Laplace-covariance construction: $O\!\big(|\mathcal{D}|\,(N{+}d)^2\big)$ to form the Hessian
at the mode and $O\!\big((N{+}d)^3\big)$ to factor/invert it for $\Sigma$---the cubic term,
which we do not hide---plus $O(SN\log N)$ for the membership samples and $O(N^2)$ to score
candidate pairs. We refit
every $8$ acquisitions rather than every step, so the amortized per-acquisition cost is one
eighth of the above. Concretely, a full $N{=}30$ run (Laplace refits plus scoring, $600$
acquisitions) takes a few seconds on one CPU core, and our $N{=}100$ pools remain
interactive; the cubic term is what would bite first at $N$ in the high hundreds, where
low-rank or block updates of $\Sigma$ (rather than full inversion) would be the natural
remedy. We do not claim to have validated $N{=}300$.

\paragraph{Posterior summaries.}
Two summaries drive the method. The returned answer is the point set
$\arg\!\topk_i\hat\theta_i$. The acquisition rule instead needs the \topk-membership
probabilities $p_i=\Pr\!\big(i\in S_k(\theta)\big)$, which have no closed form; we estimate
them by drawing $S$ samples $\theta^{(s)}\sim\mathcal{N}(\hat\theta,\Sigma)$ and counting
how often item $i$ lands in the top $k$, a Monte-Carlo estimate with $O(1/\sqrt S)$ error
($S{=}1500$ was ample for stable values). The Laplace approximation is well suited here:
the Bradley--Terry log-posterior is concave and, as comparisons accumulate, increasingly
Gaussian around the mode, so a single mode-plus-curvature fit captures the uncertainty the
acquisition needs at a fraction of the cost of MCMC.

\paragraph{Presentation order is part of the query, not an afterthought.}
Order deserves care, because $\hat\kappa$ reaches $|\hat\kappa|\approx1$--$2.4$ on some judges
and a rule that ignored it would be scoring the wrong quantity. We therefore let the
acquisition range over \emph{ordered} pairs: the candidate set is
$\{(i,j): i\neq j\}$, selecting $(i,j)$ means ``show $i$ first,'' and $\hat p_{ij}$ in
Eq.~\eqref{eq:acq} is evaluated at exactly that order, with $+\hat\kappa$ for the first-shown
item. Score and query therefore refer to the same object, and the rule may exploit
$\hat\kappa$ by choosing which item leads. This is the alternative to fixing the unordered
pair and randomizing the order afterwards---a design under which the information term would
have to be averaged over both orders,
$\tfrac12 w(i\text{ first})+\tfrac12 w(j\text{ first})$, since the two carry different Fisher
weight whenever $\kappa\neq0$; scoring one nominal order and then flipping it at random would
be incoherent. What keeps $\kappa$ identifiable is not randomization per se but that order
varies \emph{within} a pair across the collected data: our verdict matrices contain
\emph{both} orders of every unordered pair ($N(N{-}1)$ ordered queries), so disagreement
between $(i,j)$ and $(j,i)$ pins $\kappa$ down independently of $\theta$.

\section{Top-$k$-Aware Active Acquisition}
A budget should be spent where it changes the \emph{answer}---the membership of the
\topk{} set---not where it merely sharpens the global ranking. Let
$p_i = \Pr\big(i \in S_k(\theta)\big)$ be the posterior probability that item $i$ is in
the \topk{}, estimated by Monte Carlo from $\mathcal{N}(\hat\theta,\Sigma)$. Items with
$p_i$ near $0$ or $1$ are settled; items with $p_i$ near $\tfrac12$ sit on the boundary
and are where comparisons are valuable. We score a candidate ordered pair $(i,j)$ by
\begin{equation}
\begin{aligned}
s(i,j) = \;&\underbrace{\hat p_{ij}(1-\hat p_{ij})}_{\text{informativeness}}\;
\underbrace{\big(\Sigma_{ii}+\Sigma_{jj}-2\Sigma_{ij}\big)}_{\Var(\theta_i-\theta_j)}\\[-1pt]
&\times\;\underbrace{\big(\mathcal{H}(p_i)+\mathcal{H}(p_j)\big)}_{\text{boundary relevance}},
\end{aligned}
\label{eq:acq}
\end{equation}
where $\hat p_{ij}=\sig\big(\hat\theta_i-\hat\theta_j+\hat c\,(x_i-x_j)+\hat\kappa\big)$ is
the judge's \emph{predicted} response probability under the full model of
Eq.~\eqref{eq:model} (not a bias-free surrogate), $\Sigma$ is the quality block of the
\emph{joint} Laplace covariance $H^{-1}$---so it already marginalizes the nuisance
parameters $(c,\kappa)$ and their covariance with $\theta$---and $\mathcal{H}$ is the binary
entropy. The first two factors are the usual one-step uncertainty/D-optimal term---the
expected information about $\theta_i-\theta_j$; the third factor focuses that effort on
pairs whose \topk{} membership is still in doubt. Dropping the third factor recovers a
\emph{global-uncertainty} rule, which we show wastes budget on clearly-in or
clearly-out items. We greedily query $\arg\max_{(i,j)} s(i,j)$ over \emph{ordered} pairs,
reveal the judge's verdict, refit, and repeat (Alg.~\ref{alg:main}).

\begin{algorithm}[t]
\caption{Bias-Aware Bayesian Active Top-$k$. Shown with refit interval $R{=}1$ (refit every step) for clarity; our experiments refit every $R{=}8$ acquisitions.}
\label{alg:main}
\begin{algorithmic}[1]
\STATE \textbf{Input:} items with covariates $x$, budget $B$, target $k$
\STATE $\mathcal{D}\leftarrow\emptyset$
\FOR{$t=1$ to $B$}
  \STATE fit $(\hat\theta,\hat c,\hat\kappa,\Sigma)$ on $\mathcal{D}$ (Laplace)
  \STATE estimate $p_i=\Pr(i\in S_k)$ by sampling $\mathcal{N}(\hat\theta,\Sigma)$
  \STATE pick $(i,j)\leftarrow\arg\max s(i,j)$ over unqueried \emph{ordered} pairs [Eq.~\eqref{eq:acq}]
  \STATE show $i$ first, $j$ second; observe judge verdict $y$
  \STATE $\mathcal{D}\leftarrow\mathcal{D}\cup\{(i,j,\text{first}{=}i,y)\}$
\ENDFOR
\STATE \textbf{return} $\arg\!\topk_i \hat\theta_i$
\end{algorithmic}
\end{algorithm}

\paragraph{Motivation for the acquisition score.}
The score~\eqref{eq:acq} is a heuristic, not an exact expected-information rule, but it has
a clear rationale in two factors. First, adding a comparison at $(i,j)$ updates the quality
precision by the rank-one term $\pi_{ij}(1{-}\pi_{ij})(e_i-e_j)(e_i-e_j)^\top$, so by
Sherman--Morrison the posterior variance of the contrast $\theta_i-\theta_j$ shrinks by an
amount that grows with $\pi_{ij}(1{-}\pi_{ij})\,\Var(\theta_i-\theta_j)$---this is the usual
one-step information (D-optimality) term. Second, we weight that term by the
\emph{boundary-uncertainty} factor $\mathcal{H}(p_i)+\mathcal{H}(p_j)$: the binary entropy
$\mathcal{H}(p)$ is maximal when an item's \topk{}-membership probability is $p\!\approx\!\tfrac12$
(its membership is genuinely in doubt and a single comparison can flip it) and vanishes for
items that are already clearly in or out. The product thus steers comparisons to informative
\emph{and} decision-relevant pairs. Dropping the boundary factor recovers the global
D-optimal rule, which spends precision on already-settled items and is therefore wasteful
for \topk{} identification---matching the poor small-budget behavior of global-uncertainty
in Table~\ref{tab:acqbase}. We do not claim~\eqref{eq:acq} optimizes expected \topk-entropy
reduction exactly; it is a cheap boundary-targeted surrogate, and we validate its benefit
empirically.

\subsection*{Acquisition: full evaluation}\label{app:acq}
\paragraph{Analyzed at the right level.}
Fig.~\ref{fig:budget} plots \topk{} recall versus comparison budget under the bias-aware model,
and Table~\ref{tab:acqbase} tests the comparison across all five biased judges against Thompson
sampling and LUCB as well as the budget-agnostic rules. How the test is constructed matters
more here than any modeling choice, so we state it first. Each pool is run with several
acquisition seeds, but seeds within a pool share the same frozen verdict matrix, so they are
\emph{not} independent replicates: treating $10$ pools $\times\,6$ seeds as $60$ samples is
pseudo-replication and inflates significance by roughly an order of magnitude. We therefore
average seeds within a pool, treat the \textbf{pool} as the unit, use an exact paired
permutation test, and Holm-correct across baselines. Fig.~\ref{fig:acq} shows the mechanism
(our rule pours comparisons onto boundary items and starves the rest); naive variants never
exceed their wrong-\topk{} plateau regardless of budget or acquisition rule.

On the two judges with $10$ pools, the rule beats \emph{every} baseline---Thompson and LUCB
included---at $B{=}120$ (Holm $p\le.023$). On the remaining three the point estimates also
favor it, but we draw no conclusion there, for a reason worth stating precisely: with $n$ pools
an exact permutation test cannot return $p$ below $2/2^{n}$, so at $n{=}5$ ($p_{\min}{=}.063$)
and $n{=}3$ ($p_{\min}{=}.25$) significance is \emph{unobtainable by construction}. Absence of
stars in those columns measures our pool count, not the method. Two caveats bound even the
positive result: it holds under the frozen verdict matrix---a \emph{deterministic, repeat-free}
protocol---and the score is a heuristic product of three factors rather than an optimal rule.
This is why we present acquisition as a secondary, applied contribution.

\paragraph{Stronger acquisition baselines, across five judges.}
Beating round-robin and D-optimality is a low bar, so we also compare against the standard
strong strategies for \topk{} identification---\emph{Thompson sampling} (posterior boundary
sampling: draw $\theta\!\sim\!$ posterior and contest the sampled \topk{} boundary) and
\emph{LUCB} (query the weakest lower bound inside the current \topk{} against the strongest
upper bound outside)---on all five biased judges, with the \emph{pool} as the unit of analysis
(Table~\ref{tab:acqbase}). Our rule has the highest point estimate against \emph{every}
baseline---random, round-robin, global uncertainty, Thompson and LUCB---on all five judges.
Pool-level significance, however, is established only on Llama and Qwen,
the two judges with enough pools: there our rule beats \emph{every} baseline including Thompson
and LUCB (Holm-corrected $p\in[.010,.023]$). On the remaining three judges the comparison is
\emph{underpowered by construction}---with $n$ pools an exact permutation test cannot return
$p<2/2^{n}$, i.e.\ $.063$ at $n{=}5$ (4o-mini, DeepSeek) and $.25$ at $n{=}3$ (Gemini)---so we
report point estimates there but claim no significance. We therefore do not claim a general
advantage over posterior-sampling bandits: what our score adds is determinism given the
posterior (no sampling variance) and an interpretable factor structure we can ablate.

\begin{table}[t]
\centering\small
\setlength{\tabcolsep}{3pt}
\begin{tabular}{lccccc}
\toprule
$B{=}120$ recall & Llama & Qwen & 4o-mini & DeepSeek & Gemini \\
pools $n$ & $10$ & $10$ & $5$ & $5$ & $3$ \\
\midrule
\textbf{Top-$k$-aware (ours)} & $\mathbf{.81}$ & $\mathbf{.82}$ & $\mathbf{.85}$ & $\mathbf{.93}$ & $\mathbf{.88}$ \\
Thompson & $.72^{*}$ & $.77^{*}$ & $.72$ & $.88$ & $.83$ \\
LUCB & $.67^{*}$ & $.71^{*}$ & $.71$ & $.84$ & $.77$ \\
Global-uncertainty & $.60^{*}$ & $.70^{*}$ & $.68$ & $.84$ & $.67$ \\
Round-robin & $.61^{*}$ & $.65^{*}$ & $.69$ & $.79$ & $.69$ \\
Random & $.61^{*}$ & $.67^{*}$ & $.66$ & $.76$ & $.67$ \\
\midrule
min.\ attainable $p$ & $.002$ & $.002$ & $.063$ & $.063$ & $.25$ \\
\bottomrule
\end{tabular}
\caption{Acquisition rules on the five biased judges, \textbf{pool} as the unit (seeds averaged
within pool; exact paired permutation; Holm-corrected over the five baselines within each
judge). $^{*}$: $p{<}.05$ vs.\ ours (Llama $p{\in}[.010,.018]$, Qwen $[.010,.023]$). The last
row is essential for reading the right-hand columns: with $n$ pools the smallest attainable
exact $p$ is $2/2^{n}$, so on 4o-mini/DeepSeek ($n{=}5$) and Gemini ($n{=}3$) \textbf{no
pool-level test can reach $p{<}.05$}---the absence of stars there reflects unobtainable power,
not absence of effect, and point estimates favor our rule on all five judges. On the two judges
with enough pools we beat every baseline, including Thompson sampling and LUCB.}
\label{tab:acqbase}
\end{table}

\paragraph{Full-data recall is not the acquisition ceiling.}
One number deserves explicit comment because it looks paradoxical. On GPT-4o-mini the
\topk-aware rule reaches $0.89$ recall at its best budget ($B\!\approx\!200$) but the
\emph{full}-data bias-aware fit (all $870$ comparisons) is only $0.76$ (Table~\ref{tab:main}). Recall is
therefore \emph{non-monotone} in budget here: the curve peaks near $B{=}200$ and settles back
to $0.76$ by full data. This can arise under a frozen, mildly misspecified oracle
because adaptive subset selection acts as a form of implicit regularization; we treat the
non-monotonicity as a protocol-specific empirical effect rather than attributing it to a single
mechanism. The discrepancy is most pronounced on GPT-4o-mini (which carries a strong
second-position preference, $\hat\kappa{=}{-}0.84$); on Llama and Qwen the selected-budget and
full-data results are substantially closer. We therefore do not describe the rule as ``reaching
the full-data ceiling fastest'' for such judges.

\paragraph{Robustness check: a stochastic judge.}
All results above use a frozen verdict matrix (each ordered pair queried once and cached),
which suits repeat-free comparison algorithms but is not how an API judge behaves. We therefore
built a \emph{stochastic} oracle: instead of caching the argmax we store the judge's own
Bernoulli probability $p_{ij}=\sig(\text{logit}_A-\text{logit}_B)$ and draw a \emph{fresh}
verdict at every query, allowing a pair to be re-queried. The judge is genuinely stochastic
under this reading---$43\%$ of pairs have $p_{ij}\in(0.2,0.8)$---so the frozen protocol does
discard real judge uncertainty. We collected $10$ such pools (enough that an exact pool-level
permutation test can reach $p{=}.002$) and re-ran every acquisition rule with fresh verdicts,
refitting every $2$ steps; frequent refits matter here, because with a stale posterior a
deterministic greedy rule re-picks the \emph{same} pair until the next refit while Thompson's
sampling diversifies for free.

The result splits cleanly, and the split is the honest headline. Our rule leads in point
estimate at every budget ($0.64/0.71/0.78/0.83$ at $B{=}60/120/200/320$), and against the
budget-agnostic and information-theoretic rules the advantage survives the switch to a
stochastic judge: we beat round-robin at $B{=}60,120,320$ (Holm $p{=}.039/.020/.031$) and
D-optimality at $B{=}120,320$ ($p{=}.031/.020$). But against the strong bandit baselines it
does \emph{not}: we never separate from Thompson (Holm $p{=}.19$--$1.0$) or LUCB
($p{=}.27$--$1.0$) at any budget, despite leading them on points. Contrast this with the frozen
protocol, where we beat Thompson and LUCB outright (Table~\ref{tab:acqbase}). We therefore
scope the claim exactly: \textbf{targeting the \topk{} boundary beats budget-agnostic and
global-uncertainty allocation under both protocols, but its edge over posterior-sampling
bandits is an artifact of the repeat-free setting}. Designing a rule that keeps the boundary
focus while diversifying under noise---which is what Thompson gets for free---is the clearest
open problem this work leaves.

\paragraph{Ablation: the two contributions are complementary.}
A $2{\times}2$ ablation (Llama, $B{=}120$; Table~\ref{tab:ablate}, left) shows the bias-aware
model and the \topk-aware rule are \emph{synergistic}. Acquisition alone barely helps
($0.50\!\to\!0.51$ under the naive model): efficiently querying is pointless when you are
sharpening the \emph{wrong} ranking. Bias correction alone helps ($0.50\!\to\!0.61$), and the
two together reach $0.82$---the acquisition gain materializes ($+0.21$) only once the model
recovers the right quality signal. The correctness gain also holds across
$k$ (Table~\ref{tab:ablate}, right): bias-aware beats naive at $k{=}1,3,5,10$, most at
$k{=}5$ ($0.61\!\to\!0.85$). It also scales: at $N{=}100$ items the gain persists
($0.60\!\to\!0.80$ at $k{=}10$, Llama), at a runtime that stays interactive; we validate $N{\le}100$ and do not claim the method has
been demonstrated at $N$ in the hundreds (see the complexity discussion).

\begin{table}[t]
\centering\small
\begin{minipage}{0.52\columnwidth}\centering
\begin{tabular}{lcc}
\toprule
$B{=}120$ & r-robin & top-$k$ \\
\midrule
naive       & $.50$ & $.51$ \\
bias-aware  & $.61$ & $\mathbf{.82}$ \\
\bottomrule
\end{tabular}
\end{minipage}\hfill
\begin{minipage}{0.44\columnwidth}\centering
\begin{tabular}{lcc}
\toprule
$k$ & naive & bias \\
\midrule
1  & $.97$ & $\mathbf{1.00}$ \\
3  & $.79$ & $\mathbf{.90}$ \\
5  & $.61$ & $\mathbf{.85}$ \\
10 & $.82$ & $\mathbf{.87}$ \\
\bottomrule
\end{tabular}
\end{minipage}
\caption{Left: $2{\times}2$ ablation (model $\times$ acquisition), Llama, recall at
$B{=}120$---the two contributions combine super-additively. Right: recall@$k$ on full data
(mean over biased judges)---bias-aware wins at every $k$.}
\label{tab:ablate}
\end{table}

\paragraph{Acquisition-score ablation.}
The score~\eqref{eq:acq} is a product of three factors, and each earns its place (Llama,
$B{=}120$): dropping the boundary-entropy term---the \topk-aware part---costs the most
($0.76\!\to\!0.62$, collapsing to global D-optimality), while dropping the variance or the
$p(1{-}p)$ informativeness terms costs less ($0.76\!\to\!0.75$ and $0.73$); all three beat
entropy-alone ($0.67$) and random ($0.61$). The score is thus not an arbitrary product---the
boundary term (its novel ingredient) contributes most, atop two standard information factors
(Fig.~\ref{fig:ablate}).

\begin{figure}[t]
\centering
\includegraphics[width=0.80\columnwidth]{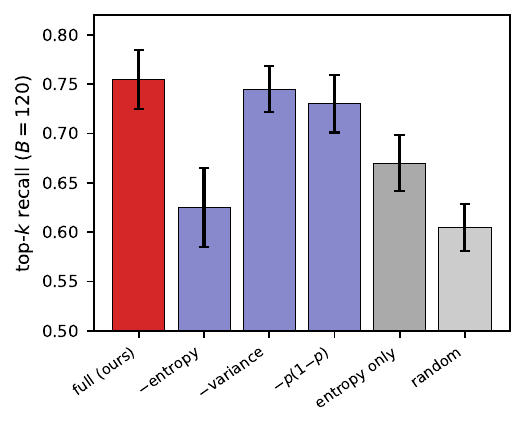}
\caption{Acquisition-score ablation (Llama, $B{=}120$, $\pm$s.e.). Each factor of the
score~\eqref{eq:acq} helps; dropping the \topk{}-boundary entropy term costs the most
(it collapses to global D-optimality). All variants beat random.}
\label{fig:ablate}
\end{figure}

\paragraph{Cost-effectiveness.}
The practical payoff is that debiasing a cheap judge beats paying for a frontier one at a
given recall (Fig.~\ref{fig:cost}).\footnote{Cost snapshot (public list prices, July 2026):
GPT-4o-mini \$0.15/\$0.60 per 1M input/output tokens; GPT-5.1 \$1.25/\$10 per 1M. A comparison
is $\approx$$293$ input $+$ $16$ output tokens ($1.3$ tokens/word). These are a snapshot for reproducibility, not a stable claim---prompt
length, output budget, hidden reasoning tokens, and vendor repricing all move the ratio, so we
report the multiplier as an instance, not a property of the method.} Pricing each comparison at
these rates (GPT-4o-mini vs.\ GPT-5.1, $\approx$$30\times$ apart on output), GPT-4o-mini with
our bias-aware model and \topk-aware acquisition
reaches recall $0.89$ for $\approx$\$$0.014$, whereas the unbiased GPT-5.1 needs
$\approx$\$$0.28$ to reach $0.93$---a $\approx$$20\times$ cost gap around $0.89$ recall at
July 2026 list prices for our $\approx$$293$-token comparisons; the multiplier is workload- and
price-dependent, not a constant. The
frontier judge \emph{does} attain perfect recall ($1.00$ at $\approx$\$$0.42$), so the last
few points of recall are worth paying for; but for the common regime of high-but-not-perfect
\topk{} recovery, correcting a cheap judge is far cheaper than upgrading the judge.

\paragraph{The crossing point is a property of the setting, not a threshold.}
Because we lean on the sign change, we must not let it read as a universal constant. It is
not. Sweeping the crossing over the factors it could plausibly depend on
(Table~\ref{tab:crossing}), it moves between $0.22$ and $0.60$. Two factors drive it: the
\emph{true} bias magnitude ($0.22$ at $c^\star{=}0.5$ rising to $0.60$ at $c^\star{=}3$---a
strongly biased judge tolerates more genuine correlation before correction backfires) and the
pool size ($0.22$ at $N{=}15$ vs.\ $0.52$ at $N{=}60$). It is comparatively insensitive to the
priors ($0.45$--$0.51$ over two decades of $\tau$; $0.46$--$0.50$ over $\lambda$), to the
quality spread ($0.42$--$0.50$), and to $k$ ($0.44$--$0.50$). What is robust is the
\emph{qualitative law}---the benefit decreases monotonically in $\mathrm{corr}(\theta,x)$ and
eventually reverses---and it is that law, not the number, that our story rests on. The
$\approx$$0.45$ figure describes our simulation configuration and should not be used as a
decision threshold on a new task; the gate, which directly assesses the operational validity of correction, is
what we recommend instead.

\begin{table}[t]
\centering\small
\setlength{\tabcolsep}{4pt}
\begin{tabular}{llc}
\toprule
factor varied & values & crossing $\mathrm{corr}(\theta,x)$ \\
\midrule
true bias $c^\star$ & $0.5 / 1 / 2 / 3$ & $\mathbf{.22 / .50 / .55 / .60}$ \\
pool size $N$ & $15 / 30 / 60$ & $\mathbf{.22 / .50 / .52}$ \\
quality spread $\sigma$ & $0.6 / 1.2 / 2.5$ & $.50 / .50 / .42$ \\
quality prior $\lambda$ & $0.2 / 1 / 5$ & $.50 / .50 / .46$ \\
covariate prior $\tau^{-2}$ & $.01 / .1 / 1$ & $.45 / .50 / .51$ \\
$k$ & $2 / 5 / 10$ & $.44 / .50 / .49$ \\
\bottomrule
\end{tabular}
\caption{Where the benefit of correction changes sign, as the simulation configuration varies
(one factor at a time from the baseline, which crosses at $0.50$). The crossing ranges over
$0.22$--$0.60$ and is governed mainly by the true bias magnitude and the pool size, not by the
priors. The monotone decrease in $\mathrm{corr}(\theta,x)$ is robust; the specific value is
not a transferable threshold.}
\label{tab:crossing}
\end{table}

\paragraph{Amortizing the anchor cost.}
That last comparison is, however, the worst case for us, because the anchor cost is paid
\emph{once per (judge, covariate, task family)} while the comparison cost is paid \emph{per pool}: the gate
asks whether this judge's verbosity preference is spurious for this task family, and that
answer is reused for every pool judged afterwards. Table~\ref{tab:amort} charges the full
$\approx$\$$0.14$ of $40$ anchors against $P$ reused pools. At $P{=}1$ the anchors dominate
($11\times$ overhead over ungated evaluation); by $P{=}10$ the overhead is $2.0\times$ and by
$P{=}100$ it is $1.1\times$, i.e.\ essentially free. The comparison that matters for a
practitioner is not gated-vs-ungated but gated-cheap-judge vs.\ simply paying for a frontier
judge, which needs no correction: the gated pipeline is already $1.8\times$ cheaper at $P{=}1$
and $18\times$ cheaper at $P{=}100$, at $0.89$ vs.\ $0.93$ recall. The caveat is that this
amortization assumes the gate's verdict transfers across the reused pools---i.e.\ that
$\mathrm{corr}(\theta,x)$ is stable within the task family; a pool drawn from a family where
verbosity becomes legitimate would need re-gating: a meaningful shift in the task family can
change $\mathrm{corr}(\theta,x)$, and the synthetic crossing near $0.45$ is illustrative rather
than a transferable drift threshold.

\begin{table}[t]
\centering\small
\setlength{\tabcolsep}{5pt}
\begin{tabular}{lccc}
\toprule
pools reusing one gate decision, $P$ & $1$ & $10$ & $100$ \\
\midrule
anchor labels ($\$0.14$ once), per pool & $\$.140$ & $\$.014$ & $\$.0014$ \\
cheap-judge comparisons, per pool & $\$.014$ & $\$.014$ & $\$.014$ \\
\midrule
\textbf{total per pool} (recall $0.89$) & $\$.154$ & $\$.028$ & $\$.0154$ \\
overhead vs.\ ungated cheap judge & $11\times$ & $2.0\times$ & $1.1\times$ \\
vs.\ frontier judge ($\$.28$, recall $0.93$) & $1.8\times$ & $10\times$ & $\mathbf{18\times}$ \\
\bottomrule
\end{tabular}
\caption{Anchor cost amortizes over reused pools. The one-time $40$-anchor gate decision is
charged against $P$ pools judged with the same judge and covariate; the last row is how much
cheaper the gated cheap judge is than simply using a frontier judge, which needs no
correction. Even at $P{=}1$ the gated pipeline is cheaper than the frontier judge; the
$10\times$ anchor overhead quoted above is a single-pool artifact.}
\label{tab:amort}
\end{table}

\paragraph{A second task domain.}
To check that the phenomenon is not an artifact of one topic, we rebuilt the pool in an
unrelated domain---\emph{Python and basic algorithms} factual QA (e.g.\ ``dictionary lookup
is on average $O(1)$'' vs.\ the false ``\dots is $O(n)$ because dictionaries are sorted
linked lists''), with the same fixed-statement-count construction and an independent
elaboration manipulation. The effect replicates and is if anything stronger: the Llama judge
shows a larger length coefficient than on the Moon pool ($\hat c{=}{+}1.74$ vs.\ $+0.99$),
naive recall is $0.46{\pm}.03$, and bias-aware recovers $0.72{\pm}.06$
($\Delta{=}{+}0.26$, $95\%$ CI $[+0.12,+0.38]$, paired $t$ $p{=}0.006$, $10$ pools). Verbosity
bias and its correction are therefore not tied to a single topic---though both pools share
the property that verbosity is spurious \emph{by construction}, which is exactly the
condition Fig.~3 shows to matter.

\paragraph{Item construction.}
Each item answers the prompt ``Give an accurate and informative overview of \emph{the
Moon}.'' from a bank of $12$ true and $12$ false statements about the topic. With a fixed
number of statements per item, quality is $\theta_i=(\#\text{true})-(\#\text{false})$; the
true \topk{} are the fully-correct items. The verbosity covariate is binary: the terse
form lists the chosen statements, the elaborated form appends to each a content-free but
authoritative-sounding clause (identical factual content, more tokens). Verbosity is
assigned by a balanced shuffle rejected unless $|\mathrm{corr}(\theta,\text{verb})|<0.12$.
We use $N{=}30$, $k{=}5$.

\paragraph{Oracle collection and offline evaluation.}
For each judge we query all $N(N{-}1)$ ordered pairs once and cache the
verdict matrix. Acquisition strategies are then evaluated \emph{offline} against this fixed
oracle: a strategy reveals matrix entries up to its budget, refits, and is scored by
\topk{} recall. This makes every strategy face the identical judge and makes results
exactly reproducible; randomness enters only through the strategy's own seed. We report
means over oracles (random item draws) and acquisition seeds ($8$ for the budget curves, $6$ for the paired significance tests).

\paragraph{Hyperparameters and protocol.}
Unless noted we use quality-prior precision $\lambda{=}1$, bias-prior precision
$\lambda_b{=}1/\tau^2{=}0.1$, $S{=}1500$ membership samples, and refit the Laplace fit every
$8$ acquisitions with warm-started L-BFGS. Items: $N{=}30$, $k{=}5$, $C{=}6$ statements per
item from a bank of $12$ true and $12$ false statements. Oracle (item-pool) counts per judge
match the $n$ column of Table~\ref{tab:main} exactly: $10$ for the locally served
Llama-3.1-8B, Qwen-2.5-Coder-7B and Phi-4 (unmetered, so we could collect more pools); $5$ for
GPT-4o-mini, DeepSeek-chat and Coder-6.7B; $3$ for Gemini-flash-lite, DeepSeek-V4-flash,
GPT-4.1, GPT-5.1 and Claude Haiku/Sonnet-4.6; $2$ for DeepSeek-V4-pro; $1$ for the
spend-capped Gemini-2.5-pro and Opus-4.8. A pool counts only if at least $80\%$ of its
$N(N{-}1){=}870$ ordered verdicts were collected (a few metered/reasoning judges time out on
individual queries); pools below this are dropped, and \textbf{GPT-5.5 is excluded entirely}
because all of its pools fell to $\approx$$30\%$ completeness. This is why some frontier judges
have $n{=}1$--$2$: a small-sample caveat we flag rather than hide. The separate Python/algorithms domain
contributes a further $10$ Llama pools, reported only in its own paragraph and never pooled
with the Moon pools. Each budget curve averages
over oracles and acquisition seeds ($8$ for curves, $6$ for significance); bands are $\pm1$ s.e. The judge prompt is: given
answers $A$ and $B$ to ``Give an accurate and informative overview of \emph{the Moon},'' ask
``Which answer is more accurate and complete? Respond with ONLY the single letter A or B,''
with $A,B$ filled by the two items in the drawn order.

\paragraph{Open-weight and closed (API) judge details.}
Open-weight judges (Llama-3.1-8B, Qwen-2.5-Coder-7B, Phi-4, DeepSeek-Coder-6.7B) are served
locally on a single 48\,GB GPU via HuggingFace transformers; the verdict is the
$\arg\max$ of the ``A''/``B'' next-token logits, which is faster than sampling and removes
decoding noise. Closed judges are queried through their official APIs with the exact ids
\texttt{gpt-4o-mini}, \texttt{gpt-4.1}, \texttt{gpt-5.1} (OpenAI);
\texttt{gemini-2.5-flash-lite}, \texttt{gemini-2.5-pro} (Google, via the OpenAI-compatible
endpoint); \texttt{deepseek-chat}, \texttt{deepseek-v4-flash}, \texttt{deepseek-v4-pro}
(DeepSeek); and \texttt{claude-haiku-4-5}, \texttt{claude-sonnet-4-6}, \texttt{claude-opus-4-8}
(Anthropic). Non-reasoning chat models use temperature $0$ and \texttt{max\_tokens}$=16$, and
we parse the first ``A''/``B'' token. Reasoning models emit hidden thinking tokens and need a
completion budget: for GPT-5.x we set \texttt{max\_completion\_tokens}$=64$ (and omit
\texttt{temperature}, which these models fix to $1$); Gemini-2.5-pro is run with
\texttt{reasoning\_effort=low} and a $2000$-token cap (it spends $\approx$$500$ hidden
reasoning tokens per verdict); DeepSeek-V4 models likewise use a $2000$-token cap. Claude
models are queried \emph{without} a temperature override (some reject temperature $\neq 1$),
with extended thinking left off, parsing the returned text block. In all cases we query \emph{both} orders of every unordered pair---all $N(N{-}1)$ ordered
queries---recording which item was shown first, and each ordered pair is queried once and
cached.

We are precise about what that caching does and does not buy, because the three settings above
are not equally deterministic. Only the \textbf{open-weight} judges are genuinely
deterministic: we read the verdict as $\arg\max$ over the next-token ``A''/``B'' logits, which
is a fixed function of the checkpoint. Among API judges, the non-reasoning chat models are
queried at temperature $0$---which makes them near-deterministic in practice but is not a
vendor guarantee---while \textbf{GPT-5.x fixes temperature to $1$} and \textbf{Claude rejects
a temperature override}, so for those judges a cached verdict is a \emph{single draw from a
stochastic judge}, not a deterministic value. Caching therefore constructs a frozen oracle
\emph{by protocol}, not because the underlying judges are deterministic. This is a deliberate
choice---it isolates the systematic bias we study from run-to-run noise---and we quantify what
it hides rather than assume it away: $43\%$ of pairs are genuinely uncertain under an open
judge's own logits, and our stochastic-judge experiment re-runs the acquisition comparison
without the freeze.

\paragraph{Reproducibility of API judges.}
The open-weight judges (Llama-3.1-8B, Qwen-2.5-Coder-7B, Phi-4, DeepSeek-Coder-6.7B) are fully
reproducible from public checkpoints, and our central claims---bias correction on
biased-but-competent judges, the $2{\times}2$ synergy, and the acquisition ablation---rest on
them. The proprietary judges were queried through their APIs in July 2026 at the exact model
ids listed above; because vendors update models behind an id, we treat those numbers as a
snapshot and do not hinge conclusions on any single frontier version. Each ordered pair is
queried once and cached, so the reported error bars reflect variation across oracles (item
pools) and acquisition seeds rather than decoding randomness---but, as noted above, this is a
property of the \emph{frozen protocol} and not a claim that the judges are deterministic: two
of the API families cannot be queried at temperature $0$ at all. A re-collection would
therefore not reproduce every cached verdict for those judges; the temperature-$1$ flip rate
we measure for GPT-4o-mini ($8\%$) is the right order of magnitude to expect.

\paragraph{Bias probes.}
The isolated-bias probes hold quality fixed and vary one factor. Position: equal-quality
pairs shown in both orders; a judge with no position bias prefers the same underlying
answer regardless of order. Verbosity/format: the same statements presented terse vs.\
elaborated, or plain vs.\ markdown. We measure $\Pr(\text{prefer the contrasted side})$
with Wilson $95\%$ intervals.

\paragraph{Synthetic simulator.}
The simulator draws $\theta_i\sim\mathcal{N}(0,1.25^2)$ and an independent verbosity
covariate, and generates verdicts from Eq.~\eqref{eq:model} with controllable
$c^\star,\kappa^\star$. It reproduces the two real-judge effects and isolates the
\emph{unbiased-judge} control ($c^\star{=}\kappa^\star{=}0$): there the bias-aware model
stays close to the naive model but pays a real price ($0.79$ vs.\ $0.86$ \topk{} recall):
the prior limits, but does not eliminate, the variance cost under an unbiased judge. This is
one reason the gate is necessary rather than optional.
The simulator also shows the global-uncertainty rule's small-budget deficit is present
even without bias, i.e.\ it is a property of \topk{} targeting, not of debiasing.

\paragraph{Shrinkage strength.}
The prior precision $1/\tau^2$ (reported as $\lambda_b$) trades off correction against
protection. Sweeping $\lambda_b\in\{0.1,1,2\}$, the bias-aware gain on biased judges
decreases gently but stays clearly positive (e.g.\ Llama drops from $+0.32$ toward $+0.24$ as
$\lambda_b$ grows). The fixed prior preserves most of the gain on biased judges, while
\emph{small negative changes remain possible} on near-neutral or sub-competent judges---the
mildly brevity-preferring Claude-Haiku ($\Delta{=}{-}0.07$) and the sub-competent Coder-6.7B
($\Delta{=}{-}0.04$) in Table~\ref{tab:main}---because modeling a covariate a judge barely has is not free.
A single moderate $\lambda_b{=}0.1$ keeps essentially all of the gain where there is bias to
correct, at this small and bounded downside where there is not. We also implemented a
\emph{hierarchical} version that learns $\tau$ per judge by empirical Bayes (alternating the
Laplace fit with the evidence update
$\tau^2\!\leftarrow\!\tfrac12(\hat c^2+\hat\kappa^2+\mathrm{Var}[\hat c]+\mathrm{Var}[\hat\kappa])$).
It removes the need to tune $\lambda_b$ and learns small $\tau$ for near-unbiased judges and
larger $\tau$ for biased ones; on our data, however, a fixed moderate $\lambda_b$ already
matches the learned-$\tau$ version, which gives no consistent
improvement (and occasionally over-shrinks). We therefore report the fixed prior in the main
results and view the hierarchical variant as a way to avoid hand-tuning $\lambda_b$ rather
than as a source of additional accuracy.

\paragraph{External-validity details.}
The LLMBar experiment (Fig.~2) judges each pair in both presentation orders
with GPT-4o-mini; the judge score is the order-averaged preference for output~1, and the
length/format features are the word-count and markdown-marker differences. The corrected
column is the $5$-fold cross-validated accuracy of an $\ell_2$-regularized logistic model
predicting the gold label from the judge score and these features. The SummEval check uses
the $100$ articles of the SummEval benchmark, taking each article's machine summaries as
items, the mean of the four human scores as ground-truth quality, and standardized word
count as the covariate (no elaboration manipulation exists in this external corpus, so the
binary indicator used on the controlled benchmark is unavailable here).

\subsection*{Proofs}
\paragraph{Proposition~\ref{prop:naive} (confounded estimand).}
Work in the Thurstone model $\Pr(a\succ b)=\Phi\!\big((\mu_a-\mu_b)/\sigma\big)$ with
augmented score $\mu_i=\theta_i+c^\star x_i$ (the position term averages out under
randomized order). The map from a score gap to a probability is a strictly increasing
bijection, so the population win-rates $\{\Pr(a\succ b)\}$ determine the gaps
$\{\mu_a-\mu_b\}$ uniquely. The naive model fits
$\Pr(a\succ b)=\Phi\!\big((\tilde\theta_a-\tilde\theta_b)/\sigma\big)$; matching all
population win-rates forces $\tilde\theta_a-\tilde\theta_b=\mu_a-\mu_b$ for every $a,b$, so
$\tilde\theta_i=\theta_i+c^\star x_i$ up to an additive constant. Under
Assumption~\ref{as:indep} the added term is non-constant across items and uncorrelated with
$\theta$, so $\tilde\theta$ is not order-equivalent to $\theta$: any $i,j$ with
$0<\theta_i-\theta_j<c^\star(x_j-x_i)$ have $\tilde\theta_j>\tilde\theta_i$. If such a pair
straddles the boundary (one truly in $S_k$, one not), then $S_k(\tilde\theta)\neq
S_k(\theta)$. As this concerns the estimand $\tilde\theta$, not a finite sample, more
comparisons cannot remove the error. The logistic case is identical with $\Phi$ replaced by
the (also bijective) logistic link. \hfill$\square$

\paragraph{Proposition~\ref{prop:ident} (non-identifiability), full argument.}
Write the design vector of a comparison as $z_{ab}=(e_a-e_b,\;x_a-x_b,\;1)$. The covariate
entry satisfies the identity
\begin{equation}
x_a-x_b=\sum_i x_i\,(e_{a,i}-e_{b,i}),
\label{eq:collinear}
\end{equation}
i.e.\ the $c$ column of the design is the fixed linear combination $\sum_i x_i\cdot(\text{item
column }i)$ of the $\theta$ columns, \emph{for every comparison, regardless of which pairs are
observed}. Hence $v_\delta=(\delta x_1,\dots,\delta x_N,\,-\delta,\,0)$ satisfies
$z_{ab}^\top v_\delta=\delta(x_a-x_b)-\delta(x_a-x_b)=0$ for all $a,b$: it is a null direction
of the whole design, so $\beta$ and $\beta+v_\delta$ give identical likelihoods. This is
exactly the reparametrization $\theta_i'=\theta_i+\delta x_i,\;c'=c-\delta$ of
Eq.~\eqref{eq:invariance}. Adding same-$x$ comparisons, cross-$x$ comparisons, both orders, or
arbitrarily many repeats cannot remove $v_\delta$, because~\eqref{eq:collinear} holds
comparison-by-comparison; no design over a fixed item pool with one free $\theta_i$ per item
escapes it. The Hessian $H=\sum_{\mathcal D}\pi(1{-}\pi)z z^\top$ is therefore singular along
$v_\delta$ (in addition to the usual Bradley--Terry constant $(\mathbf 1,0,0)$), which we
confirm numerically: rank $30$ for a $32$-column design. Only $\phi=\theta+cx$ (up to a
constant) and $\kappa$---whose column is made independent of the item block by order
randomization---are identified. Adding the Gaussian prior makes $H+\Lambda\succ0$ and hence
the MAP unique, but this is the prior selecting a point along $v_\delta$, and the selected
point is given by~(3); it is not the data identifying $c$. \hfill$\square$

\paragraph{Why the paired design escapes the argument.}
The proof of~\eqref{eq:collinear} uses that each item has both its own free $\theta_i$
\emph{and} a single fixed $x_i$. In the paired design the parameter attached to item $a$ is
$\theta_{b(a)}$ for its \emph{base content} $b(a)$, and two items with the same base carry
different covariate values. The would-be null direction now requires
$\delta x_a=\delta x_{a'}$ for two items $a\neq a'$ sharing a base but with
$x_a\neq x_{a'}$, forcing $\delta=0$. Concretely, a within-base comparison has
$e_{b(a)}-e_{b(a')}=0$ but $x_a-x_{a'}\neq0$, so its logit is $c(x_a-x_{a'})+\kappa$ and
depends on $c$ alone. The only remaining null direction is the Bradley--Terry constant, which
matches the observed rank of $16$ out of $17$ columns. \hfill$\square$

\paragraph{Paired-design experiment details.}
Each base content is a fixed-length list of $6$ statements drawn from the same true/false bank
as the main benchmark, so its quality is $\theta=(\#\text{true})-(\#\text{false})$ exactly as
before; the terse rendering lists the statements and the verbose rendering appends an
elaboration clause to each, leaving the statement count---and therefore $\theta$---unchanged.
Each base thus contributes two items sharing one $\theta$ at two covariate values, which is
what breaks the invariance. We collect the \emph{full} ordered verdict matrix
($N(N{-}1)$ comparisons, $N{=}2\times$bases) from each judge. The main paired study is $4$
open judges $\times$ $6$ pools at $15$ bases ($24$ pools, $20{,}880$ verdicts); the design-size
sweep adds $4$ judges $\times$ $3$ seeds at $8$, $25$ and $40$ bases ($108{,}120$ further
verdicts, the largest pools being $80$ items and $6{,}320$ comparisons each). The two analyses
in Table~\ref{tab:paired} differ \emph{only} in the design matrix built from identical data:
the unpaired analysis gives each rendering its own $\theta$ column, the paired analysis maps
both renderings of a base to one column. Profile likelihoods $\max_{\theta,\kappa}\ell$ at
fixed $c$ are computed by L-BFGS on the unpenalized objective, and the reported range is over
a $9$-point grid $c\in[\hat c{-}1,\hat c{+}1]$.

\paragraph{Acquisition score (derivation of the heuristic).}
A comparison at $(i,j)$ with Fisher weight $w=\pi_{ij}(1{-}\pi_{ij})$ updates the quality
precision by $w\,uu^\top$, $u=e_i-e_j$. By Sherman--Morrison the posterior variance of the
contrast $v=\theta_i-\theta_j$ strictly decreases, by an amount that grows with $w$ and with
the current $\Var(v)$; this is the D-optimal / one-step information term
$\pi_{ij}(1{-}\pi_{ij})\Var(v)$ in the score. We multiply it by a
\emph{boundary-uncertainty} weight $\mathcal H(p_i)+\mathcal H(p_j)$, where $p_m$ is the
posterior probability that item $m$ is in the \topk{}. The binary entropy $\mathcal H(p)$ is
maximal at $p{=}\tfrac12$ and zero at $p\in\{0,1\}$, so this weight is large exactly for
items whose membership is still in doubt---the items a single comparison can move across the
\topk{} boundary---and negligible for items already clearly in or out. We emphasize that
$\mathcal H(p)$ is used as a boundary-uncertainty weight, \emph{not} as an approximation to
the entropy-reduction sensitivity $|\mathcal H'(p)|$ (which instead vanishes at $p{=}\tfrac12$);
the score is therefore a motivated surrogate rather than an exact expected-entropy-reduction
rule. Dropping the boundary weight recovers the pure D-optimal rule, which is
information-optimal for $v$ but indifferent to whether $v$ affects \topk{} membership,
explaining its poor small-budget behavior in Table~\ref{tab:acqbase}. \hfill$\square$

\subsection*{Expanded related work}\label{app:related}
\paragraph{Covariates, fixed effects, and identification in measurement models.}
The mathematics behind our Prop.~\ref{prop:ident} is not new, and we want to be explicit
about that rather than let it be discovered. Our immediate problem is exact collinearity in the design matrix, and hence
non-identification---formally distinct from the Neyman--Scott
\emph{incidental-parameter problem}, which concerns the inconsistency of a structural
parameter as nuisance parameters proliferate with the sample. But it is a close relative, and
belongs to the broader fixed-effects/incidental-parameter literature
\citep{neyman1948consistent,lancaster2000incidental}: both arise from granting each unit a free
effect, and here those effects do not merely contaminate the covariate coefficient---they
absorb it entirely, leaving the likelihood flat. The remedies
we invoke are equally standard---condition on, or design for, \emph{within-unit} variation, as
in conditional logit for panel data with fixed effects \citep{chamberlain1980analysis} and
conditional maximum likelihood for the Rasch model \citep{andersen1970asymptotic}. In that
language, our paired rendering design is a within-item repeated-measures design, and it works
for exactly the reason those methods work. Measurement theory has likewise long separated
rater effects from ability \citep{linacre1989many}, studied when an item behaves differently
across groups \citep{holland2012differential}, and formalized when an instrument may be
compared across populations at all \citep{meredith1993measurement}.

We therefore claim no theorem that would surprise a psychometrician. What we claim is
that this known impossibility silently governs a practice the ML community is
building on: LLM-judge debiasing is performed and reported as if the covariate coefficient
were estimated, and the field has neither noticed that it cannot be, nor asked what the
resulting numbers mean. Our contribution is the transfer and its consequences---measuring
that the flat direction is exact and data-proof on real judges ($0.0000$ nats, $48$ pools),
showing that reported ``debiasing'' is a prior-induced decomposition governed by
$\widehat{\mathrm{Cov}}(\phi,x)$ (Eq.~3), characterizing when that assumption pays
and when it silently destroys signal, and giving a gate and a design that make the choice
operational for practitioners who are not going to derive conditional likelihoods.

\end{document}